%% file: main.tex
\documentclass{article} 
\usepackage{iclr2027_conference,times}

\input{math_commands.tex}

\usepackage{hyperref}
\usepackage{url}
\usepackage{amsthm}
\usepackage{amsmath}
\usepackage{amssymb}
\usepackage{mathtools}
\usepackage{xspace}
\usepackage{wrapfig}
\usepackage{tabularx}
\usepackage{booktabs}
\usepackage{adjustbox}
\usepackage{tikz}
\usetikzlibrary{arrows.meta}
\usetikzlibrary{positioning}
\theoremstyle{plain}
\newtheorem{theorem}{Theorem}[section]
\newtheorem{proposition}[theorem]{Proposition}
\newtheorem{lemma}[theorem]{Lemma}
\newtheorem{corollary}[theorem]{Corollary}
\theoremstyle{definition}

\usepackage[most]{tcolorbox}
\usepackage{xcolor}         
\usepackage[dvipsnames]{xcolor}
\usepackage[table]{xcolor}
\usepackage{graphicx}
\usepackage{wrapfig}
\usepackage{subcaption}
\usepackage{multirow}
\usepackage{pifont}

\newcommand{\cmark}{\textcolor{ForestGreen}{\ding{51}}}
\newcommand{\xmark}{\textcolor{BrickRed}{\ding{55}}}
\newcommand{\circlednum}[1]{%
\tikz[baseline=(char.base)]{
\node[shape=circle,fill=black,text=white,inner sep=1.2pt,font=\scriptsize\bfseries] (char) {#1};
}}

\newcommand{\methodlong}{\underline{\textbf{t}}ractable \underline{\textbf{r}}epresentations for pre\underline{\textbf{i}}mage learning and inverse \underline{\textbf{o}}ptimization\xspace}
\newcommand{\method}{\mbox{TRIO}\xspace}
\newcommand{\Methodlong}{\underline{\textbf{T}}ractable \underline{\textbf{R}}epresentations for Pre\underline{\textbf{i}}mage Learning and Inverse \underline{\textbf{O}}ptimization\xspace}

\title{Representation Learning for Exact Preimages}

\author{
\textbf{Konstantin Hess}\textsuperscript{1,2,*},
\textbf{Stefan Feuerriegel}\textsuperscript{1,2}\\[0.8em]
\textsuperscript{1}LMU Munich \quad
\textsuperscript{2}Munich Center for Machine Learning \quad
\\
\textsuperscript{*}{Corresponding author: \texttt{k.hess@lmu.de}}
}

\iclrfinalcopy
\begin{document}

\maketitle

\begin{abstract}

Modern neural predictors can model highly nonlinear maps, but many scientific and engineering tasks require reasoning in the opposite direction: given a performance or safety level, the goal is to characterize the preimage, that is, the complete set of inputs which meet the desired target level and optimize over that set. For expressive neural predictors, however, such preimages typically have no explicit representation and are expensive to recover or optimize over. This creates a fundamental three-way challenge between {expressive forward prediction}, {accurate preimage approximation}, and {tractable optimization over the preimage for downstream tasks}. We introduce \method~({\methodlong}), a framework for learning representations that make these objectives compatible by construction. Our key contribution is a \emph{preimage factorization}: the forward model remains expressive through nonlinear radial transformations (including neural networks), while, under inversion, each transformation reduces to a single scalar radius, which yields simple geometric level sets. This yields an explicit geometric representation that is reusable for downstream optimization over the preimage, and, for linear objectives, we show that this admits a closed-form global solution. We finally prove a universal approximation theorem which shows that \method can approximate any continuous forward map and its entire family of potentially disconnected, nonconvex preimages arbitrarily well. Hence, \method combines expressive forward modeling, exact preimage recovery, and tractable global downstream optimization over preimages by design.

\end{abstract}

\begin{wrapfigure}{r}{0.48\textwidth}
\vspace{-1.3cm}
\centering
\includegraphics[width=0.48\textwidth, trim=1.8cm 17.5cm 3.4cm 2.5cm, clip]{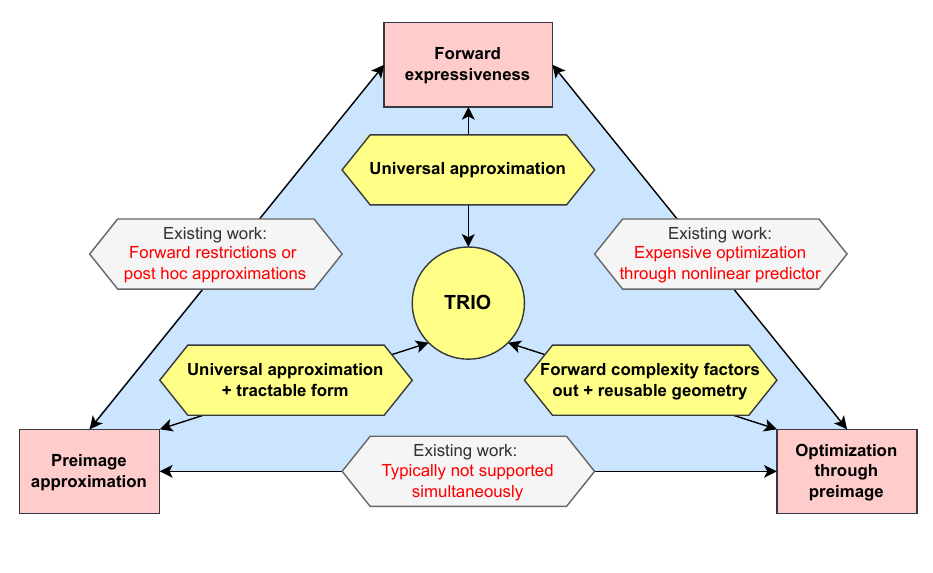}
\vspace{-0.8cm}
\caption{\textbf{Filling the gap:} \method combines forward expressiveness with exact preimage approximation and tractable inverse optimization.}
\label{fig:contribution}
\vspace{-0.4cm}
\end{wrapfigure}

\section{Introduction}\label{sec:intro}
Many problems in science and engineering require not only predicting the outcome of a given input, but instead identifying the full set of inputs that satisfy a desired performance or safety requirement. For example, in inverse design, safety, and backward reachability, the goal is to find all designs that meet a prescribed target level under a learned surrogate, and then optimize over the resulting preimage~\citep{bhosekar2018advances,kotha2023provably,lee2023inverse}.

However, this type of backward reasoning is difficult for expressive predictors: the forward map may be easy to evaluate, while its preimages are typically not available in explicit form and costly to recover or optimize over~\citep{ceccon2022omlt,hoang2024surrogates,kotha2023provably}. This creates a \emph{fundamental three-way challenge between expressive forward prediction, accurate preimage approximation, and tractable downstream optimization over the preimage}. Existing work has largely addressed these challenges along three separate directions: (i)~recovering preimages post hoc from trained predictors \citep[e.g.,][]{dathathri2019inverse,kotha2023provably,bjorklund2025premap2}; (ii)~expensive inverse optimization through learned predictors \citep[e.g.,][]{bergman2022janos,ceccon2022omlt,schweidtmann2019deterministic}; and (iii) designing restricted model architectures that make inverse reasoning more tractable~\citep[e.g.,][]{amos2017icnn,behrmann2019invertible,papamakarios2021flows}. 
Yet, an approach that addresses all challenges simultaneously is still missing (see Figure~\ref{fig:contribution}).  

We introduce \method~(\methodlong), a framework that combines expressive forward prediction with exact preimage representation and tractable downstream optimization. For any target chosen after training, \method represents the complete learned preimage as an explicit finite union of ellipsoidal regions, even when the predictor is non-injective and the preimage is disconnected or nonconvex. Unlike approaches that recover preimages post hoc, repeatedly optimize over the predictor, or impose convexity or invertibility for tractability, \method provides an explicit preimage representation that can be reused for set queries and downstream optimization.

The key idea behind \method is a \emph{preimage factorization}: \method uses expressive radial transformations for the forward prediction, while their complexity factors out when computing the preimage. Specifically, in the backward pass, \method combines radial components through a union operation, while each radial transformation collapses into a single scalar radius. The result is an exact finite geometric preimage even when the radial transformations are parametrized by neural networks. Importantly, this construction separates forward expressiveness from the complexity of preimage computation.

Our contributions are fivefold\footnote{Code is available at \url{https://github.com/konstantinhess/TRIO_representations}.}: 
\textbf{(1)}~We introduce \method, a framework that unifies expressive forward prediction, accurate preimage approximation, and tractable downstream optimization over the preimage. \textbf{(2)}~We prove that the preimage factorization is exact: every target level yields a finite union of ellipsoidal regions, while the complexity of the forward radial models factors out of subsequent preimage reasoning. \textbf{(3)}~We show that this geometry enables globally exact downstream optimization, and, for linear objectives, this even admits a closed-form global solution. \textbf{(4)}~We establish a universal approximation theorem to show that \method can approximate any continuous forward map and the complete family of target-dependent preimages arbitrarily well. \textbf{(5)}~We show empirically that \method learns accurate forward predictors and preimages while substantially accelerating exact inverse optimization over baselines.

\section{Related Work}\label{sec:rw}
Prior work has largely studied the tension between expressive forward prediction, preimage approximation, and tractable downstream optimization along three three separate directions that we review below: (i)~recovering preimages from trained predictors, (ii)~solving inverse optimization problems through learned predictors, and (iii)~designing architectures that make inverse reasoning more tractable.

\textbf{(i) Preimage approximation:}~\emph{How can we recover the whole inverse set of a trained network?}
For a non-injective predictor, a given target level generally corresponds \emph{not} to a single input but to an entire, potentially disconnected set of inputs. Recovering this set from a generic neural network is computationally challenging. Early methods construct under- and over-approximations by propagating output constraints backward through the network~\citep{dathathri2019inverse}. More recent approaches refine these approximations through branch-and-bound~\citep{kotha2023provably}, relaxations~\citep{zhang2025premap,bjorklund2025premap2}, probabilistic approximation~\citep{marzari2026probabilistic}, or targeted refinement procedures~\citep{koller2026shadows}. 
$\Longrightarrow$  \emph{These methods aim to approximate the preimage post hoc at substantial computational cost; \method has an explicit representation of the learned preimage by construction.}

\textbf{(ii)~Inverse optimization subject to learned predictors:}~\emph{How can we optimize over inputs when a trained predictor determines the constraint?}
A direct approach is to embed the learned forward model into the downstream optimization problem and search over its inputs~\citep{bergman2022janos,ceccon2022omlt}. Depending on the model class, this may involve gradient-based nonlinear optimization~\citep{wachter2006implementation}, mixed-integer formulations for piecewise-linear networks~\citep{anderson2020strong,tsay2021partition}, or deterministic global optimization through relaxations and spatial branch-and-bound~\citep{schweidtmann2019deterministic,bestuzheva2025scip}. Across these approaches, the feasible set remains defined implicitly by the learned predictor, so each new downstream objective requires solving a new complex optimization problem through the nonlinear model. 
$\Longrightarrow$ \emph{These methods optimize over a feasible set defined implicitly by the predictor; \method instead represents this set explicitly, which removes the nonlinear predictor from subsequent downstream optimization.}

\textbf{(iii)~Customized architectures for tractable inverse reasoning:}~\emph{Can the predictor itself be designed to make backward reasoning tractable?}
Input convex neural networks (ICNNs) impose convexity to enable efficient downstream optimization~\citep{amos2017icnn,pfrommer2023asymmetric}, but this restricts their sublevel sets to convex geometry. Invertible neural networks~\citep{jacobsen2018irevnet,behrmann2019invertible} impose bijectivity, so each output corresponds to a single input, which is inapplicable in settings where many inputs can attain the same output. Normalizing flows similarly rely on invertible transformations~\citep{dinh2017realnvp,papamakarios2021flows}; when used for inverse reasoning, they can represent distributions over possible inputs, but do not provide the complete feasible preimage as an explicit set. 
$\Longrightarrow$ \emph{These methods enforce convexity or bijectivity for tractability reasons; \method allows for non-injective, non-convex mappings and disconnected preimages}.

\textbf{Research gap:} Methods that jointly address (i)--(iii) are missing. To fill this gap, \method learns expressive predictors with exact, explicit, and reusable learned preimages, while preserving universal forward and backward approximation, and reducing downstream global optimization to computationally tractable optimization over ellipsoidal regions. 

\section{Setup}\label{sec:setup}

\begin{wrapfigure}{r}{0.4\textwidth}
\vspace{-1.6cm}
\centering
\includegraphics[width=0.4\textwidth, trim=3.5cm 20.5cm 9.5cm 3.5cm, clip]{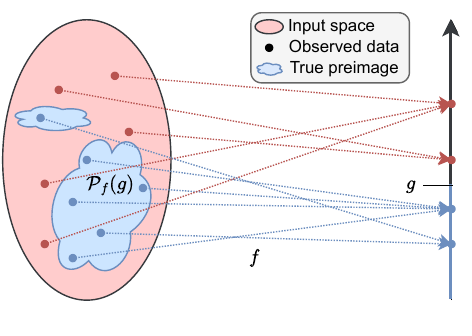}
\vspace{-0.8cm}
\caption{\textbf{Computing the preimage:} \method allows for non-injective, non-convex mappings and disconnected preimages.}
\label{fig:preimage}
\vspace{-0.8cm}
\end{wrapfigure}

\textbf{Notation:} Let $\mathcal{X}\subseteq\mathbb{R}^d$ be the admissible input domain, and let $f:\mathcal{X}\rightarrow\mathbb{R}$ denote an unknown, potentially nonlinear and non-injective function. We observe training data $\mathcal{D}=\{(x_i,y_i)\}_{i=1}^N$, with $y_i=f(x_i)$.

For a target level $g\in \overline{\mathbb{R}}=\mathbb{R}\cup \{-\infty, \infty\}$ and any scalar map $h:\mathcal{X}\rightarrow\mathbb{R}$, we define its sublevel set
\begin{align}
\mathcal{P}_h(g)
\coloneqq
\{x\in\mathcal{X}:h(x)\le g\}
=
h^{-1}((-\infty,g]).
\end{align}

Thus, $\mathcal{P}_h(g)$ gives the preimage of the lower output set $(-\infty,g]$. We refer to $\mathcal{P}_h(g)$ simply as the \emph{preimage} at level $g$ throughout the paper. The sublevel sets are nested by definition; i.e., $g_1\le g_2$ implies $\mathcal{P}_h(g_1)\subseteq\mathcal{P}_h(g_2)$. Any such preimage mapping induces a Galois connection, which has a natural order-theoretic interpretation (see Supplement~\ref{sec:galois}).

\textbf{Problem statement:} We aim to learn a predictor $F_\theta$ that supports three tasks: (i)~accurate forward prediction of $f$, (ii)~approximation of the corresponding preimages ${P}_f(g)$, and (iii)~tractable downstream optimization over these preimages. 
Importantly, we make \emph{no injectivity or convexity assumption}; hence, a preimage may be nonconvex and disconnected.

Given an objective $J:\mathcal{X}\rightarrow\mathbb{R}$ (e.g., cost or risk), we aim to learn the following tasks:
\begin{tcolorbox}[colback=White!8,colframe=Black!75!black,boxrule=0.8pt,arc=1.5mm,left=1.2mm,right=1.2mm,top=1mm,bottom=1mm,fonttitle=\bfseries,coltitle=white,enhanced,breakable] 
\vspace{-0.3cm}
\begin{align}
    \circlednum{1}\;\underbrace{f:\mathcal X \to \mathbb R,}_{\text{forward model}} \qquad  \qquad \circlednum{2}\;\underbrace{\mathcal{P}_f(g) \text{ for } g\in \mathbb R}_{\text{preimage}} \qquad 
    \circlednum{3}\;\underbrace{\min_{x\in\mathcal{X}} J(x) \; \text{s.t.} \; f(x)\le g }_{\text{downstream optimization}}.
\end{align}
\end{tcolorbox}
An equivalent objective for the latter is $\min \{J(x): \; x\in\mathcal{P}_f(g)\}$. These tasks are common in inverse design, safety, and surrogate-based optimization~\citep{kotha2023provably,ceccon2022omlt}.

Our aim is to learn a predictor that combines \emph{(i) expressive forward modeling} with \emph{(ii) exact preimage representation}, and \emph{(iii) tractable downstream optimization over the preimage}. Here, $F_\theta$ captures the nonlinear structure of $f$, while its induced preimages $\mathcal{P}_{F_\theta}(g)$ admit exact\footnote{Throughout, \emph{exact} refers to the learned predictor $\widehat{\mathcal{P}}(g)=\mathcal{P}_{F_\theta}(g)$. This structural guarantee is different from agreement with the unknown ground-truth preimage $\mathcal{P}_f(g)$.}, explicit, and reusable representations that support efficient querying and globally exact downstream optimization. \emph{Hence, unlike prior work, we aim to do so (i)~\textbf{without} post-hoc preimage recovery or separate inverse models, (ii)~\textbf{without} repeatedly optimizing through the learned predictor, and (iii)~\textbf{without} imposing invertibility or convexity on the forward map.}

\section{\method}

We now introduce \method and provide our main theoretical properties. \textbf{Section~\ref{sec:prefact}} defines the predictor class, and shows that, for every target level, the corresponding preimage has an exact finite representation in which expressive forward transformations factor out, and therefore removes the complexity during subsequent backward reasoning. \textbf{Section~\ref{sec:optimization}} uses the preimage representation to obtain tractable global downstream optimization over the learned preimage, and shows that this yields closed-form global solutions for linear objectives. Finally, \textbf{Section~\ref{sec:universal_approximation}} establishes a universal approximation theorem, thus showing \method can approximate any continuous forward function and the entire resulting family of preimages arbitrarily well.

\subsection{Exact preimage factorization with \method}
\label{sec:prefact}
\textbf{Key idea:} The key to \method is a \emph{preimage factorization}: We parameterize forward expressiveness through radial transformations, such that its complexity factors out when computing the preimage.

\begin{wrapfigure}{r}{0.6\textwidth}
\vspace{-0.5cm}
\centering
\includegraphics[width=0.6\textwidth, trim=0.5cm 16.6cm 0.5cm 3cm, clip]{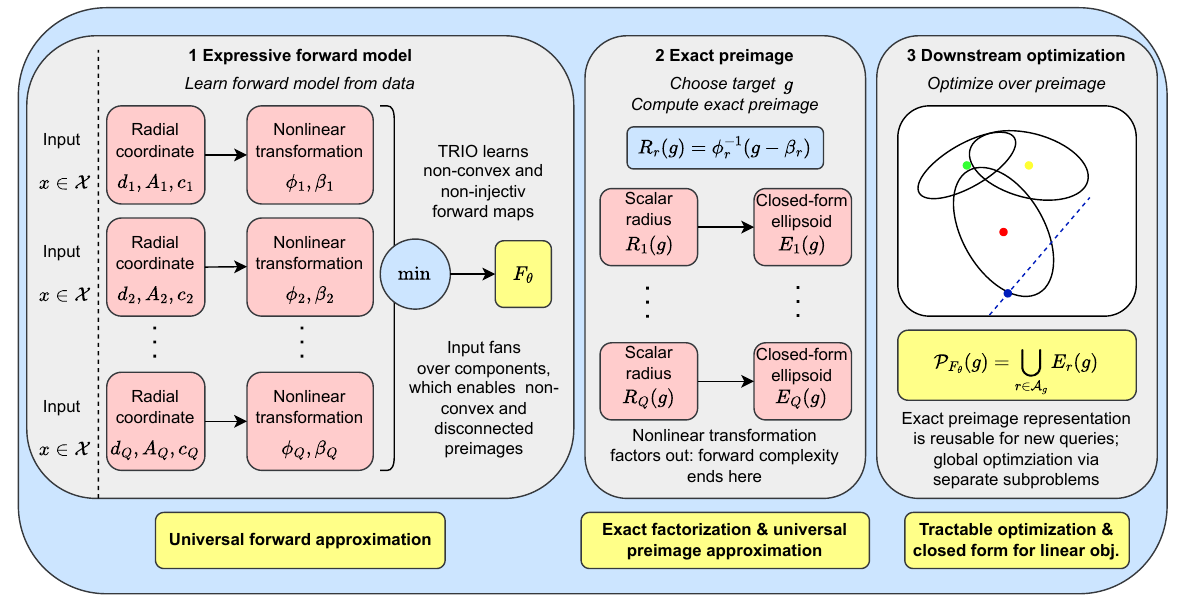}
\vspace{-0.4cm}
\caption{\textbf{\method:} The forward predictor is learned from data; for a chosen target, the preimage can be obtained excatly, which enables downstream optimziation over independent subproblems.}
\label{fig:method}
\vspace{-0.4cm}
\end{wrapfigure}

Instead of using an unrestricted predictor for which preimages must later be recovered post hoc, we parameterize the forward model through {expressive nonlinear transformations of learned radial coordinates with explicitly controlled geometry}. Here, monotonicity along each radial coordinate ensures that, for a fixed target level, each transformation reduces to a simple scalar radius. Therefore, they completely factor out of the spatial preimage representation.

For a fixed target level, each radial component induces an ellipsoidal sublevel set. In the backward operation, the minimum envelope over these components then turns into a union of the corresponding ellipsoids, which allows \method to represent disconnected and nonconvex preimages. This construction preserves flexibility in the forward model while providing an exact and reusable representation of its preimages.


\begin{tcolorbox}[
colback=White!8,
colframe=Green!75!black,
boxrule=0.8pt,
arc=1.5mm,
left=2mm,
right=2mm,
top=1.5mm,
bottom=1.5mm,
fonttitle=\bfseries,
coltitle=white,
title={\Methodlong (\method)},
enhanced,
breakable
]

\method defines
\vspace{-.6cm}
\begin{align}
F_\theta(x)
=
\min_{r=1,\ldots,Q}
\left[
\beta_r+\phi_r(d_r(x))
\right],
\label{eq:prefact}
\end{align}
for $Q$ radial components, where each component applies a radial transformation $\phi_r$ along the learned radial coordinate
\begin{align}
d_r(x)
\coloneqq
\sqrt{(x-c_r)^\top A_r(x-c_r)}.
\label{eq:spd_distance}
\end{align}

\noindent

\medskip

The bias terms are given by $\beta_r\in\mathbb{R}$; each radial transformation $\phi_r:[0,\infty)\rightarrow[0,\infty)$ satisfies
$\phi_r(0)=0$ and is continuous, unbounded, and strictly increasing; the locations are $c_r\in\mathbb{R}^d$; and the spatial geometries are controlled by $A_r\succ0$. The outer minimum couples the transformations into a nonconvex and non-injective predictor which allows for disconnected preimages.
\end{tcolorbox}

Intuitively, the construction provides flexibility along two complementary axes. First, each component learns its own center and anisotropic geometry, and the outer minimum allows different components to dominate in different parts of the input space. Their target-dependent sublevel sets therefore combine into expressive unions of ellipsoidal regions, which allows for disconnected and nonconvex preimages. Second, along each learned radial coordinate $\phi_r$, each radial component can use a flexible nonlinear transformation $\phi_r$, including neural networks, without complicating the backward geometry.

The above specification has direct implications for the inversion step: The preimage structure of \method is determined entirely by the predictor, \textbf{not} by a specialized training procedure. All parameters are learned jointly from the input data $\{(x_i,y_i)\}_{i=1}^N$ using standard empirical risk minimization such as mean squared error. Hence, \emph{no inverse model, target-specific training, or expensive post-hoc preimage computation is required.}

\textbf{Factoring out forward complexity in preimage computation:}
For a fixed target level $g$, the minimum envelope in \Eqref{eq:prefact} turns the preimage into a union over components, while monotonicity reduces each radial transformation $\phi_r$ to a single target-level-dependent radius. The following theorem shows that this factorization is exact.

\begin{tcolorbox}[colback=White!8,colframe=Green!75!black,boxrule=0.8pt,arc=1.5mm,left=1.2mm,right=1.2mm,top=1mm,bottom=1mm,fonttitle=\bfseries,coltitle=white,enhanced,breakable] 
\begin{theorem}[Exact preimage factorization]\label{thm:exact_prefact}
For any target level $g$, we define the active  components $\mathcal{A}(g)\coloneqq\{r:g\geq\beta_r\}$,
and, for each $r\in\mathcal{A}(g)$, the radius
$R_r(g)\coloneqq\phi_r^{-1}(g-\beta_r)$. Then, the learned preimage is exactly
\begin{align}
\mathcal P_{F_{\theta}}(g) = 
\bigcup_{r\in\mathcal{A}(g)}
E_r(g)
\end{align}
with
$
E_r(g)
\coloneqq
\left\{
x\in\mathcal{X}:
(x-c_r)^\top A_r(x-c_r)
\leq
R_r(g)^2
\right\}
$. 
Hence, every target level induces $g$ an exact finite union of at most $Q$ ellipsoidal regions.
\end{theorem}
\begin{proof}
See Supplement~\ref{sec:proofs}.
\end{proof}
\end{tcolorbox}

Theorem~\ref{thm:exact_prefact} holds for any parameter values satisfying the conditions above, and is therefore independent of training accuracy or generalization. It shows is that the complexity of the forward radial transformations does not carry over to computing the preimage.  Each active radial transformation enters the preimage only through the scalar
\begin{align}
R_r(g)
=
\phi_r^{-1}(g-\beta_r) .
\end{align}
Hence, once the radius $R_r(g)$ is computed, $\phi_r$ no longer appears in the preimage representation. The learned preimage is therefore fully specified by the geometric parameters $\big(c_r,A_r,R_r(g)\big)$.

\begin{wraptable}{r}{0.5\textwidth}
\vspace{-0.3cm}
\setlength{\intextsep}{0pt}
\setlength{\columnsep}{1em}
\centering
\begin{adjustbox}{width=\linewidth}
\tiny
\begin{tabular}{lll}
\toprule
\textbf{Radial model $\phi_r$}
&
\textbf{Forward parameterization}
&
\textbf{Evaluation of $\phi_r^{-1}$}
\\
\midrule
(Broken) Power
&
(Piecewise) power law
&
(Piecewise) closed form
\\
Monotone spline
&
Rational-quadratic spline
&
Analytic within a spline bin
\\
Monotone neural
&
Strictly monotone neural network
&
One-dimensional root solve
\\
\bottomrule
\end{tabular}
\end{adjustbox}
\caption{\textbf{Radial instantiations.} Different choices of $\phi_r$ change the forward parameterization while preserving the same ellipsoidal preimage representation.}
\label{tab:instantiations}
\vspace{-0.3cm}
\end{wraptable}

\textbf{Computation of radii:} The only required step in constructing the preimage is therefore the computation of the radii $R_r(g)$ through simple scalar inversion of $\phi_r(g)$. \emph{The entire preimage is then available analytically.} 

For neural parameterizations of $\phi_r^{-1}(g)$, strict monotonicity allows $R_r(g)$ to be computed by simple one-dimensional root-finding problem with a unique solution. Once the radii have been computed for a given target level, the resulting preimage representation can be reused across arbitrarily many downstream queries without reevaluating the radial transformations; we formalize this as an \emph{amortization property} in Supplement~\ref{sec:amorized_factorization}. Importantly, more expressive choices of $\phi_r$ (see Table~\ref{tab:instantiations}) may change the forward model and the one-time radius computation, but not the form or spatial complexity of the resulting preimage representation.

$\Longrightarrow$~\emph{\method represents each learned preimage exactly as a finite union of ellipsoidal geometries  that can be reused for subsequent downstream optimization over the preimage. The radial transformations $\phi_r$ may be arbitrarily expressive, but, after a single scalar inversion, their complexity factors out, and subsequent preimage queries operate only on the resulting geometric regions.}


\subsection{Efficient optimization over learned preimages}\label{sec:optimization}

The preimage factorization of \method also has implications for the downstream optimization task
\begin{align}
\min \left\{J(x): \; x\in\mathcal{P}_{F_\theta}(g)\right\}.
\label{eq:downstream_problem}
\end{align}
For a generic neural predictor, the feasible set $F_\theta(x)\leq g$ is defined only implicitly by the learned model and is generally nonconvex. For \method, Theorem~\ref{thm:exact_prefact} replaces this implicit feasible set with the exact finite union
\begin{align}
\mathcal P_{F_\theta}(g)
=
\bigcup_{r\in\mathcal{A}(g)}
E_r(g),
\end{align}
and therefore
\begin{align}
\inf_{x\in\mathcal P_{F_\theta}(g)}J(x)
=
\min_{r\in\mathcal A(g)}
\inf_{x\in E_r(g)}J(x).
\label{eq:regionwise_optimization}
\end{align}
Thus, global optimization over the generally nonconvex learned preimage is solved by optimizing separately over regions $E_r(g)$, and selecting the best solution. 
This is useful when the optimization  problems over each region are tractable. For example, if the input space $\mathcal X$ is convex, every active ellipsoid $E_r(g)$ is convex as well since $A_r\succ0$. If, in addition, $J$ is convex, then optimization over each region in~\Eqref{eq:regionwise_optimization} is a convex optimization problem. Therefore, global optimization over the full nonconvex preimage reduces
to finitely many convex optimization tasks.

The above reduction is especially relevant for linear objectives. For a generic nonlinear predictor $F_\theta$ (e.g., an MLP), even a linear downstream objective $J$ can still lead to a highly difficult nonconvex inverse-constrained problem as the predictor imposes the feasibility constraint. Under \method, \emph{complexity of $F_\theta$ has already been factored out}, and, when $\mathcal X=\mathbb R^d$, linear optimization over each ellipsoidal region has a closed-form solution. Evaluating these solutions across all active regions yields the global optimum over the complete, potentially disconnected preimage.

\begin{tcolorbox}[
colback=White!8,
colframe=Green!75!black,
boxrule=0.8pt,
arc=1.5mm,
left=1.2mm,
right=1.2mm,
top=1mm,
bottom=1mm,
fonttitle=\bfseries,
coltitle=white,
enhanced,
breakable
]
\begin{theorem}[Closed-form linear optimization]
\label{thm:linear_optimization}
Let $\mathcal{X}=\mathbb{R}^d$, let $a\in\mathbb{R}^d$ with $a\neq0$, and fix a target level $g$ with $\mathcal{A}(g)\neq\varnothing$. Let $J(x)=a^\top x$ be a linear objective function. Then
\begin{align}
\inf_{x\in\mathcal{P}_{F_\theta}(g)}
J(x)
=
\min_{r\in\mathcal{A}(g)}
\left[
a^\top c_r
-
R_r(g)\sqrt{a^\top A_r^{-1}a}
\right].
\label{eq:linear_optimum}
\end{align}
If $r^\star$ attains the minimum above, then a globally optimal solution is
\begin{align}
x^\star
=
c_{r^\star}
-
\frac{R_{r^\star}(g)}
{\sqrt{a^\top A_{r^\star}^{-1}a}}
A_{r^\star}^{-1}a.
\label{eq:linear_optimizer}
\end{align}
Thus, global linear optimization over the generally nonconvex, disconnected preimage reduces to evaluating one scalar objective value for each active radial component.
\end{theorem}
\begin{proof}
See Supplement~\ref{sec:proofs}.
\end{proof}
\end{tcolorbox}

Once the output-level-dependent radii are available, Theorem~\ref{thm:linear_optimization} removes the inverse-constrained optimization entirely. In other words, no iterative nonlinear solver or mixed-integer representation of the predictor is required. Instead, global optimality follows from evaluating the analytic solution associated with each active region, and then selecting the best one.


\subsection{Universal approximation theorem}
\label{sec:universal_approximation}

A natural question is whether the exact preimage structure of \method limits its approximation power. The following result shows that it does not: a single finite \method predictor can approximate any continuous forward function and the corresponding preimages arbitrarily well.

\begin{tcolorbox}[
colback=White!8,
colframe=Green!75!black,
boxrule=0.8pt,
arc=1.5mm,
left=1.2mm,
right=1.2mm,
top=1mm,
bottom=1mm,
fonttitle=\bfseries,
coltitle=white,
enhanced,
breakable
]
\begin{theorem}[Universal approximation for forward predictor and preimage representation]
\label{thm:universal}
Let $\mathcal{X}\subset\mathbb{R}^d$ be compact and let
$f\in C(\mathcal{X})$. Then, for every $\varepsilon>0$, there exists a
finite \method predictor $F_\theta$ such that
\begin{align}
\sup_{x\in\mathcal{X}}
\lvert F_\theta(x)-f(x)\rvert
<
\varepsilon.
\end{align}
and, simultaneously for every target level $g\in\mathbb{R}$,
\begin{align}
\mathcal{P}_f(g)
\subseteq
\mathcal P_{F_{\theta}}(g)
\subseteq
\mathcal{P}_f(g+\varepsilon)
.
\label{eq:universal_filtration}
\end{align}
The result already holds for the restricted subclass
$F_\theta(x)
=
\min_{r=1,\ldots,Q}
\left[
\beta_r+\lambda\|x-c_r\|_2^2
\right]$,
where $\lambda>0$ is a common quadratic coefficient shared by all experts.
\end{theorem}
\begin{proof}
See Supplement~\ref{sec:proofs}.
\end{proof}
\end{tcolorbox}

Theorems~\ref{thm:exact_prefact} and~\ref{thm:universal} establish two complementary guarantees. Theorem~\ref{thm:exact_prefact} shows that every \method preimage is represented \emph{exactly} by a finite union of ellipsoidal regions. Theorem~\ref{thm:universal} shows that imposing this exact structure does \emph{not} restrict what the model can ultimately learn. Instead, a single finite predictor can approximate any continuous forward map arbitrarily well, and, at the same time, its preimage representation learns the entire family of ground-truth preimages to arbitrarily fine resolution. Thus, the learned preimages are both \emph{exact for the learned predictor $F_\theta$} (Theorem~\ref{thm:exact_prefact}) and \emph{arbitrarily faithful to the ground-truth preimages} of the underlying function~(Theorem~\ref{thm:universal}).

Surprisingly, the above universal approximation theorem requires neither expressive radial transformations nor learned anisotropic metrics; it already holds for the restricted subclass with $A_r=I$ and $\phi_r(d)=\lambda d^2$. Richer choices of $A_r$ and $\phi_r$ therefore provide additional flexibility for finite $Q$, while preserving the same exact preimage factorization. In particular, the target level $g$ need not be specified during training. Here, Theorem~\ref{thm:universal} shows that the same learned predictor supports preimage queries over a continuum of target levels \emph{without} changing the model or the form of its preimage representation.

\textbf{Approximation rate:} For Lipschitz targets, we further show that \method achieves a constructive uniform approximation rate of $\mathcal O(Q^{-1/d})$ in the number of radial components $Q$; see Supplement~\ref{sec:approximation_rate}.


\section{Experiments}
\label{sec:experiments}

\begin{wrapfigure}{r}{0.5\textwidth}
    \centering
\vspace{-1.5cm}
    \begin{subfigure}[t]{0.98\linewidth}
        \centering
        \includegraphics[width=\linewidth]{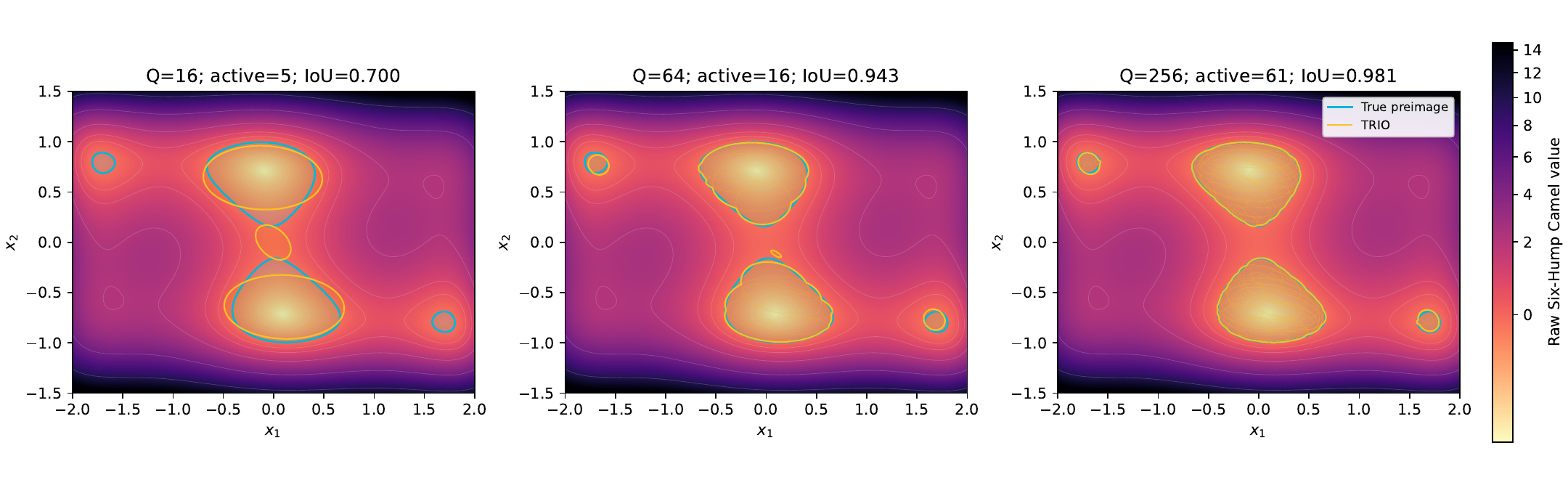}
    \vspace{-0.6cm}
        \caption{Target level $g=-0.1$.}
        \label{fig:camel_qual_top}
    \end{subfigure}

    \vspace{0cm}

    \begin{subfigure}[t]{0.98\linewidth}
        \centering
        \includegraphics[width=\linewidth]{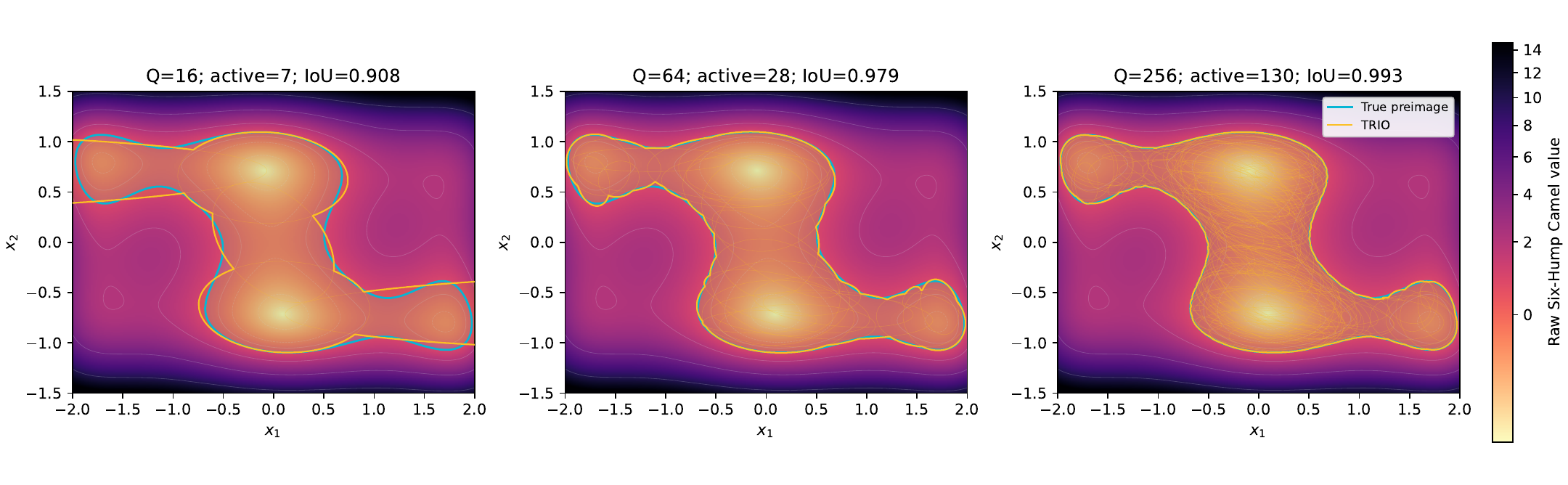}
    \vspace{-0.6cm}
        \caption{Target level $g=0.9$.}
        \label{fig:camel_qual_bottom}
    \end{subfigure}
    \vspace{-0.2cm}
    \caption{
    \textbf{Preimage reconstruction.}
    For fixed neural backbone, larger \(Q\) yields increasingly accurate geometric reconstruction, while the number of active experts varies across target levels.
    }
    \label{fig:camel_qualitative}
    \vspace{-0.3cm}
\end{wrapfigure}

The aim of our experiments is to validate the theoretical properties of \method empirically. Specifically, we analyze whether \emph{(i)} \method retains sufficient expressiveness to achieve accurate forward prediction; \emph{(ii)} \method learns accurate preimages while preserving their exact and explicit representation; and \emph{(iii)} \method allows for efficient downstream optimization over the preimage. For a fair comparison, all neural baselines have the same number of parameters.

We compare \method against baselines corresponding to the above aims. \textit{(i)}~For forward prediction, we use standard unrestricted MLPs, together with an input convex neural network (ICNN)~\citep{amos2017icnn} as a baseline with tractable convex sublevel sets. \textit{(ii)}~For preimage recovery, we compare against PREMAP2~\citep{bjorklund2025premap2,zhang2025premap} as a state-of-the-art post-hoc reconstruction method. \textit{(iii)}~For downstream optimization, we compare against standard approaches that optimize directly through the learned predictor: IPOPT~\citep{wachter2006implementation} and SCIP~\citep{bestuzheva2025scip} for smooth neural predictors (smooth MLPs), Gurobi~\citep{gurobi2026} for continuous piecewise-linear networks (CPWL MLPs), and CLARABEL~\citep{goulart2026clarabel} for ICNNs. All results are reported over ten seeds. Details about the datasets, training, and solver are in \textbf{Supplement~\ref{sec:experiment_supplement}}.

\begin{wrapfigure}{r}{0.33\textwidth}
\vspace{-0.6cm}
\centering
\includegraphics[width=0.33\textwidth, trim=0.25cm 0.cm 0.25cm 0.cm, clip]{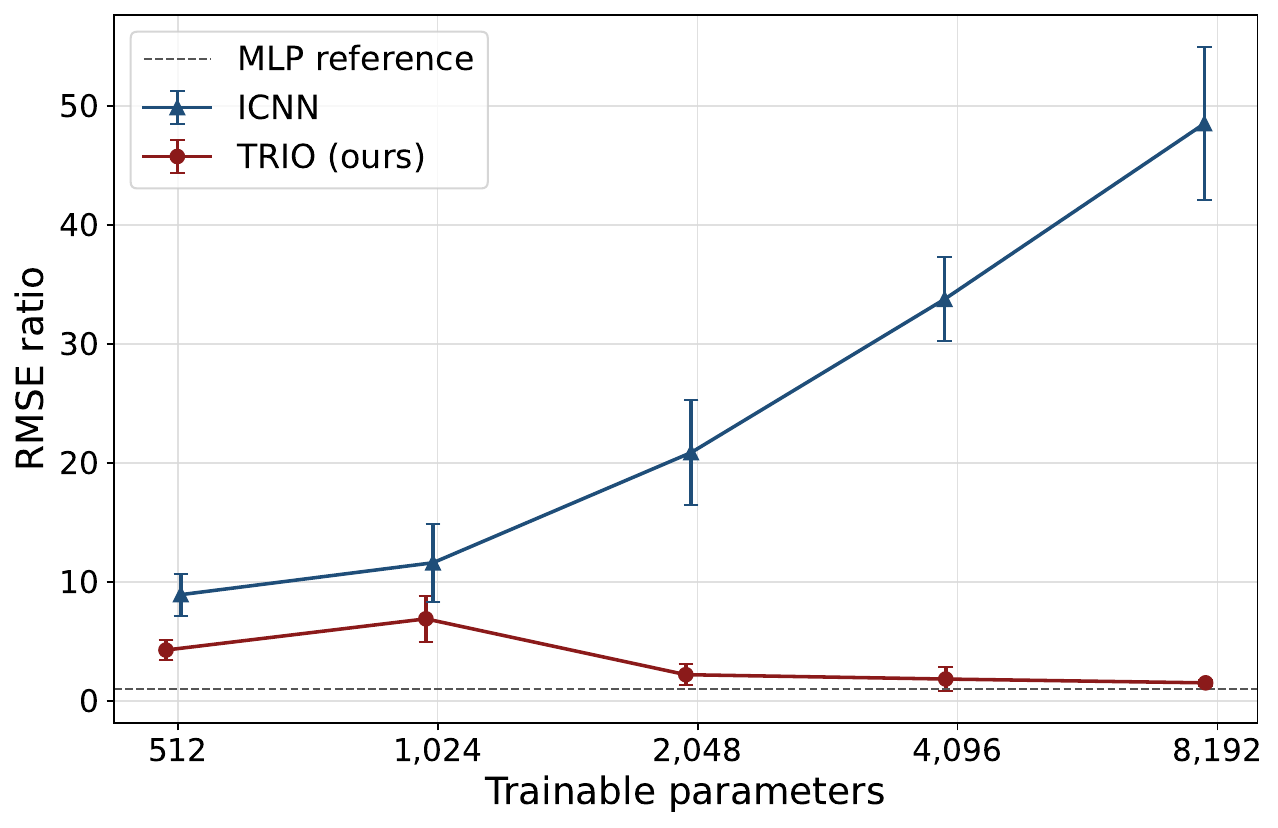}
\vspace{-0.4cm}
\caption{\textbf{Forward RMSE:} For increasing parameters, \method converges quickly to the forward RMSE of the MLP, while the ICNN completely deteriorates.}
\label{fig:forward_ratio}
\vspace{-0.8cm}
\end{wrapfigure}

\textbf{\underline{Six-Hump Camel dataset.}}~The Six-Hump Camel data provides a two-dimensional benchmark with strongly nonconvex and disconnected sublevel sets (details in Supplement~\ref{app:dgp_data}). 


$\bullet$~\emph{What do the learned preimages look like?}
We use a neural radial backbone and vary only the number of radial components $Q$, thereby allowing us to isolate the effect of geometric capacity. Figure~\ref{fig:camel_qualitative} shows how the learned preimage changes as $Q$ increases from $16$ to $256$ for a fixed radial backbone. Here, we see that learned preimage progressively approximates the true boundary and can recover finer nonconvex structures. This is consistent with the intended behavior of \method: increasing $Q$ provides a finer preimage representation, while different target levels $g$ activate different components of the same trained predictor.

$\bullet$~\emph{How does forward expressiveness improve as the model grows?} In Figure~\ref{fig:forward_ratio}, we increase $Q$ and show how forward RMSE of \method compares against that of an unrestricted MLP using the same number of parameters. For this, we report the ratio of the respective RMSEs. As the parameter count increases, the RMSE of \method approaches quickly approaches the MLP, while a same-sized ICNN leads to a large performance gap. Thus, increasing $Q$ improves not only the preimage representation but also the forward predictive performance of \method, without imposing global convexity. This behavior is consistent with our universal approximation theorem.

\begin{wraptable}{r}{0.6\textwidth}
\vspace{-0.3cm}
\setlength{\intextsep}{0pt}
\setlength{\columnsep}{1em}
\centering
\begin{adjustbox}{width=\linewidth}
\tiny
    \begin{tabular}{llrrrr}
        \toprule
        Method
        & Budget
        & Certified coverage $\uparrow$
        & Truth IoU $\uparrow$
        & Time (s) $\downarrow$
        & Rel. time increase
        \\
        \midrule

        \multirow{6}{*}{PREMAP2}
        & \cellcolor{gray!15}$64$
        & \cellcolor{red!15}$0.42\pm0.03$
        & \cellcolor{red!15}$0.42\pm0.03$
        & \cellcolor{red!15}$7.42\pm0.24$
        & \cellcolor{red!15} ${1\,178\times}$
        \\

        & \cellcolor{gray!15}$128$
        & \cellcolor{red!15}$0.63\pm0.02$
        & \cellcolor{red!15}$0.63\pm0.02$
        & \cellcolor{red!15}$14.51\pm0.99$
        & \cellcolor{red!15} ${2\,303\times}$
        \\

        & \cellcolor{gray!15}$256$
        & \cellcolor{red!15}$0.74\pm0.02$
        & \cellcolor{red!15}$0.74\pm0.02$
        & \cellcolor{red!15}$28.12\pm1.90$
        & \cellcolor{red!15} ${4\,463\times}$
        \\

        & \cellcolor{gray!15}$512$
        & \cellcolor{red!15}$0.78\pm0.03$
        & \cellcolor{red!15}$0.77\pm0.03$
        & \cellcolor{red!15}$53.83\pm3.54$
        & \cellcolor{red!15} ${8\,545\times}$
        \\

        & \cellcolor{gray!15}$1\,024$
        & \cellcolor{red!15}$0.80\pm0.03$
        & \cellcolor{red!15}$0.79\pm0.03$
        & \cellcolor{red!15} $116.86\pm29.96$
        & \cellcolor{red!15} ${18\,549\times}$
        \\

        & \cellcolor{gray!15}$2\,048$
        & \cellcolor{red!15}$0.80\pm0.03$
        & \cellcolor{red!15}$0.80\pm0.03$
        & \cellcolor{red!15}$228.29\pm35.57$
        & \cellcolor{red!15} ${36\,237\times}$
        \\

        \midrule

        \textbf{\method}~(ours)
        & \cellcolor{green!15}{exact}
        & \cellcolor{green!15}${\bm{1.00\pm0.00}}$
        & \cellcolor{green!15}${\bm{0.96\pm0.02}}$
        & \cellcolor{green!15}${\bm{0.01\pm0.00}}$
        & \cellcolor{green!15}${1\times}$
        \\

        \bottomrule
    \end{tabular}
    \end{adjustbox}

    \vspace{-0.2cm}
    \caption{
        \textbf{Exact preimage vs. post-hoc recovery.}
        \method represents the learned preimage by construction, achieves a higher ground-truth IoU, and is up to $10^5$ times faster. 
    }
    \label{tab:premap_comparison}
    \vspace{-0.3cm}
\end{wraptable}

$\bullet$~\emph{Exact versus post-hoc preimage recovery.}
We now assess  the advantage of an exact preimage representation over post-hoc recovery. To do so, we compare the explicit preimage from \method against post-hoc recovery using PREMAP2 for an unrestricted MLP. Table~\ref{tab:premap_comparison} shows the benefits. \method provides the exact preimage by construction, thus achieves almost perfect coverage (i.e., itersection over union [IoU] of $0.96$ compared to the true preimage) and requires only $0.01$\,s. For larger refinement budget, PREMAP2 reaches only $0.80$ IoU, and, at $B=2048$, the post-hoc recovery requires even $228.29$\,s, (=$3.6\times10^4$ times the cost of \method). This shows the benefit of \method: the exact preimages \emph{by construction} lead to almost perfect coverage of the learned preimage and computational speedups of several orders of magnitude.

\textbf{\underline{Increasing nonlinearity}.} We next analyze whether the properties of \method hold when the underlying data becomes increasingly nonlinear. Here, we use a dataset based on complex powers and thus highly non-injective maps indexed by $p$, where larger $p$ induce increasingly oscillatory forward dynamics and increasingly multimodal, disconnected preimages. We instantiate \method with the Broken-Power radial backbone and compare against a smooth MLPs, CPWL MLPs, and ICNNs. Experimental details are provided in the Supplement~\ref{app:baselines}.

$\bullet$~\emph{Can \method maintain (i)~forward accuracy and (ii)~accurate preimage approximation as nonlinearity increases?}
We measure (i)~forward accuracy by test RMSE and (ii)~backward accuracy by preimage IoU across target levels. The neural baselines do \textbf{\underline{not}} provide an explicit preimage natively; we approximate their preimage using a grid-based forward evaluation and output thresholding.

\begin{table*}[h!]
    \vspace{-0.1cm}
    \centering
    \begin{adjustbox}{max width=\textwidth}
    \begin{tabular}{lcccccccccc}
        \toprule
        \multicolumn{1}{l}{Complex power}
        & \multicolumn{2}{c}{$p=4$}
        & \multicolumn{2}{c}{$p=6$}
        & \multicolumn{2}{c}{$p=8$}
        & \multicolumn{2}{c}{$p=10$}
        & \multicolumn{2}{c}{$p=12$}
        \\
        \cmidrule(lr){2-3}
        \cmidrule(lr){4-5}
        \cmidrule(lr){6-7}
        \cmidrule(lr){8-9}
        \cmidrule(lr){10-11}

        Model
        & RMSE $[\downarrow]$ & IoU $[\uparrow]$
        & RMSE $[\downarrow]$ & IoU $[\uparrow]$
        & RMSE $[\downarrow]$ & IoU $[\uparrow]$
        & RMSE $[\downarrow]$ & IoU $[\uparrow]$
        & RMSE $[\downarrow]$ & IoU $[\uparrow]$
        \\
        \midrule

        Smooth MLP
        & \cellcolor{green!15}{$\bm{0.17\pm0.05}$} &  \cellcolor{green!15}{$\bm{0.99\pm0.00}^\dagger$}
        & \cellcolor{red!15}$0.56\pm0.18$ & \cellcolor{red!15}$0.95\pm0.01^\dagger$
        & \cellcolor{red!15}$0.83\pm0.21$ & \cellcolor{red!15}$0.90\pm0.02^\dagger$
        & \cellcolor{red!15}$3.15\pm6.23$ & \cellcolor{red!15}$0.79\pm0.15^\dagger$
        & \cellcolor{red!15}$3.18\pm5.94$ & \cellcolor{red!15}$0.75\pm0.12^\dagger$
        \\

        CPWL MLP
        & \cellcolor{red!15}$0.68\pm0.09$ & \cellcolor{red!15}$0.98\pm0.00^\dagger$
        & \cellcolor{red!15}$1.37\pm0.18$ & \cellcolor{red!15}$0.92\pm0.01^\dagger$
        & \cellcolor{red!15}$2.26\pm0.30$ & \cellcolor{red!15}$0.85\pm0.01^\dagger$
        & \cellcolor{red!15}$3.89\pm1.54$ & \cellcolor{red!15}$0.76\pm0.02^\dagger$
        & \cellcolor{red!15}$4.61\pm0.60$ & \cellcolor{red!15}$0.68\pm0.01^\dagger$
        \\

        ICNN
        & \cellcolor{red!15}$31.49\pm0.24$ & \cellcolor{red!15}$0.38\pm0.02^\dagger$
        & \cellcolor{red!15}$26.67\pm0.34$ & \cellcolor{red!15}$0.39\pm0.02^\dagger$
        & \cellcolor{red!15}$23.59\pm0.27$ & \cellcolor{red!15}$0.40\pm0.03^\dagger$
        & \cellcolor{red!15}$21.27\pm0.20$ & \cellcolor{red!15}$0.40\pm0.03^\dagger$
        & \cellcolor{red!15}$19.68\pm0.36$ & \cellcolor{red!15}$0.37\pm0.05^\dagger$
        \\

        \midrule

        \textbf{\method}~(ours)
        & \cellcolor{red!15} $0.19\pm0.01$ & \cellcolor{green!15}{$\bm{0.99\pm0.00}$}
        & \cellcolor{green!15}{$\bm{0.31\pm0.02}$} & \cellcolor{green!15}{$\bm{0.97\pm0.00}$}
        & \cellcolor{green!15}{$\bm{0.51\pm0.03}$} & \cellcolor{green!15}{$\bm{0.94\pm0.00}$}
        & \cellcolor{green!15}{$\bm{0.78\pm0.10}$} & \cellcolor{green!15}{$\bm{0.88\pm0.01}$}
        & \cellcolor{green!15}{$\bm{1.01\pm0.12}$} & \cellcolor{green!15}{$\bm{0.84\pm0.01}$}
        \\

        \bottomrule
        \multicolumn{11}{l}{\cellcolor{yellow!15}$\dagger$: {The preimages are \underline{\textbf{not}} native to the baseline models, but are instead approximated by computationally expensive dense grid evaluation and output thresholding.}}
    \end{tabular}
    \end{adjustbox}
    \vspace{-0.2cm}
    \caption{
        \textbf{Forward accuracy and preimage approximation:}
        Forward RMSE ($\times 10^{-2}$; $\downarrow$ lower is better) and preimage IoU
        ($\uparrow$ higher is better), reported as mean $\pm$ std. dev. over ten seeds.
    }
    \vspace{-0.2cm}
    \label{tab:complex_forward}
\end{table*}

Table~\ref{tab:complex_forward} reports the results: (i)~\method achieves the best forward accuracy from $p=6$ onward. (ii)~In terms of preimage approximation, \method achieves the highest IoU across all values of $p$, with a more pronounced advantage for larger $p$. The MLPs and ICNN are not a benchmark in itself but a performance comparison for predictors without the ability of generating preimages. Overall, this confirms the desired behavior by \method: it maintains both accurate forward prediction and accurate preimage approximation even for highly nonlinear settings.

$\bullet$~\emph{Does the explicit preimage representation translate into more efficient downstream optimization?} Here, we solve, for each learned predictor and target level $g$, the projection problem
$\min_x \frac{1}{2}\|x-x_0\|_2^2
\text{ s.t. }
F_\theta(x)\le g$,
which is the smallest perturbation required to enter the learned preimage. For the neural baselines, the constraint remains implicit in the predictor and is handled using global smooth nonlinear optimziation (SCIP for smooth MLP), global mixed integer optimization (Gurobi for CPWL MLP), local smooth nonlinear optimization (IPOPT for smooth MLP), or convex optimization (CLARABEL for ICNN). In contrast, \method \emph{optimizes directly over the explicit preimage without additional solvers} (see Supplement~\ref{app:inverse_optimization}).

\begin{figure}[h]
    \centering

    \begin{subfigure}[t]{0.32\linewidth}
        \centering
        \includegraphics[width=\linewidth]{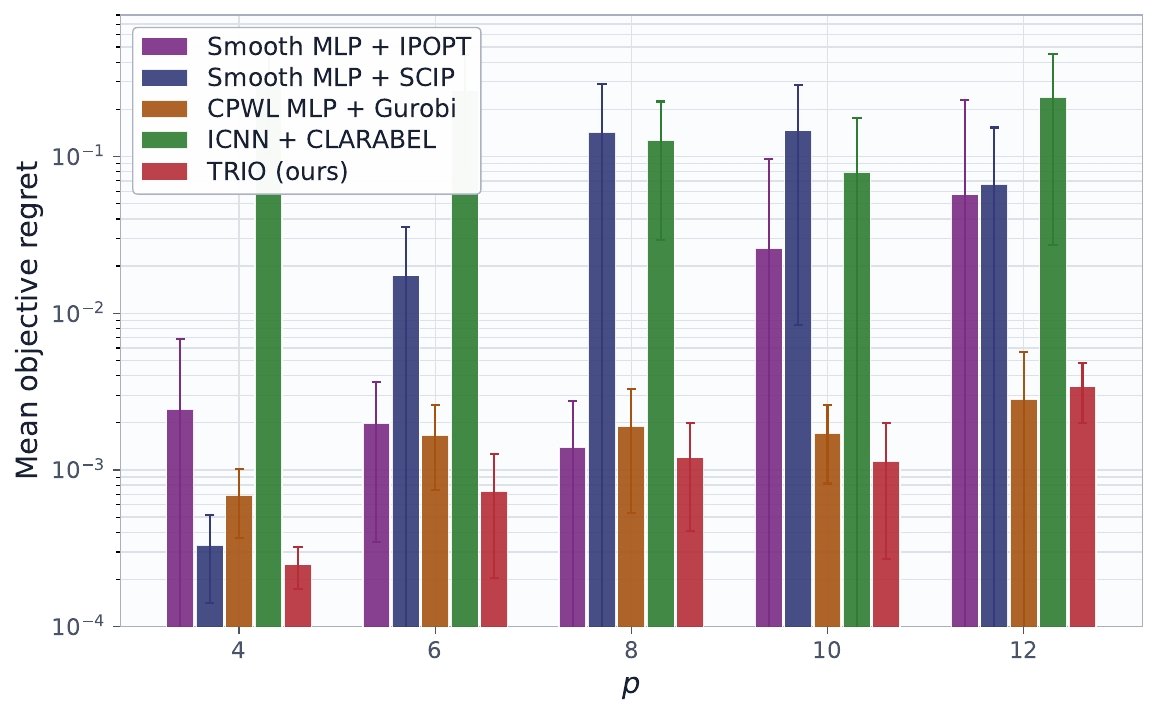}
        \vspace{-0.5cm}
        \caption{$[\downarrow]$ Objective regret}
        \label{fig:sub_b}
    \end{subfigure}
    \hfill
    \begin{subfigure}[t]{0.32\linewidth}
        \centering
        \includegraphics[width=\linewidth]{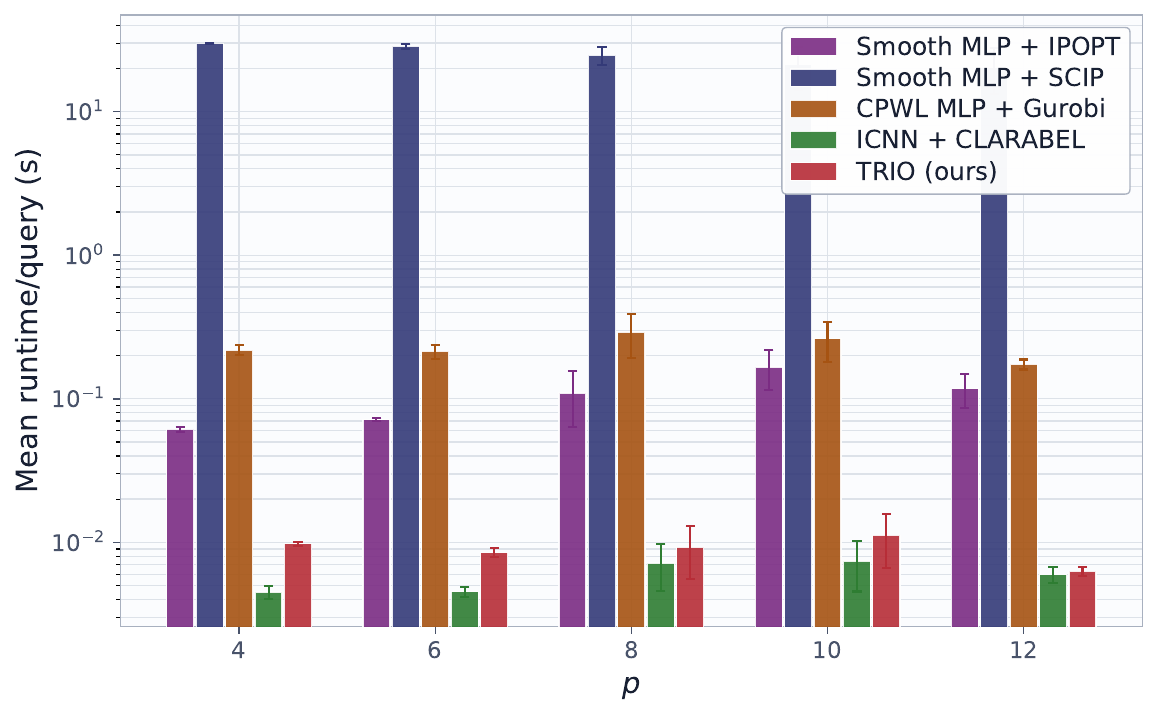}
        \vspace{-0.5cm}
        \caption{$[\downarrow]$ Runtime/query (s)}
        \label{fig:sub_c}
    \end{subfigure}
    \hfill
    \begin{subfigure}[t]{0.32\linewidth}
        \centering
        \includegraphics[width=\linewidth]{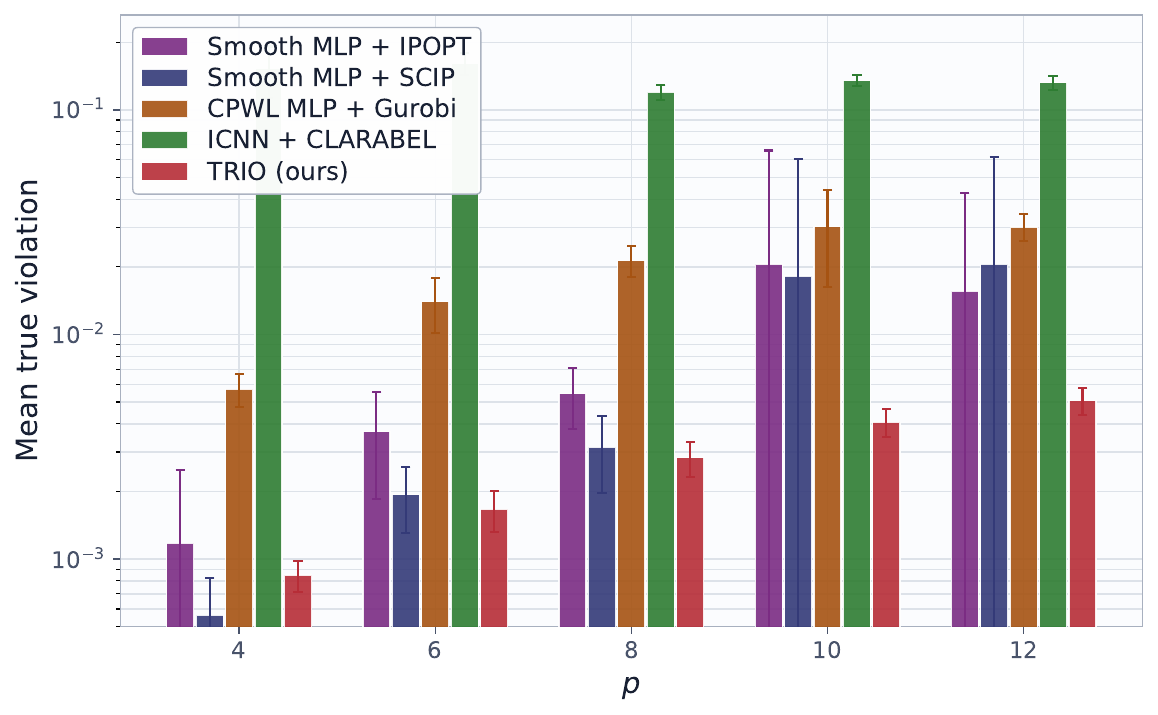}
        \vspace{-0.5cm}
        \caption{$[\downarrow]$ Constraint violation}
        \label{fig:sub_a}
    \end{subfigure}

    \vspace{-0.15cm}
    \caption{
    \textbf{Downstream optimization.}
    \textbf{(a)} Regret wrt. objective ($\downarrow$ lower is better);
    \textbf{(b)} Runtime per projection query ($\downarrow$ lower is better); and
    \textbf{(c)} Constraint violation ($\downarrow$ lower is better).
    }
    \label{fig:main_1x3}
    \vspace{-0.5cm}
\end{figure}

Figure~\ref{fig:main_1x3} shows the results: (i)~\method achieves the lowest regret wrt. objective; (ii)~it achieves substantially lower per-query runtimes than optimization through the unrestricted neural predictors; 
and (iii)~it achieves the lowest constraint violation. Overall, these results show that the explicit preimages from \method translates into accurate and computationally efficient downstream optimization.

\begin{wraptable}{r}{0.5\textwidth}
\vspace{-0.3cm}
\setlength{\intextsep}{0pt}
\setlength{\columnsep}{1em}
\centering
\begin{adjustbox}{width=\linewidth}
\tiny
    \begin{tabular}{lrrrr}
        \toprule
        Radial backbone
        & Parameters
        & One-time inversion (ms)
        & Time/query (ms)
        & Time/ellipse ($\mu$s)
        \\
        \midrule

        Broken-Power
        & $1\,152$
        & $0.04\pm0.02$
        & \cellcolor{green!15}$13.68\pm5.40$
        & \cellcolor{green!15}$106.84\pm42.21$
        \\

        Spline
        & $3\,200$
        & $6.32\pm2.52$
        & \cellcolor{green!15}$13.92\pm5.84$
        & \cellcolor{green!15}$108.78\pm45.63$
        \\

        Neural-small
        & $5\,120$
        & $13.08\pm5.08$
        & \cellcolor{green!15}$13.82\pm5.65$
        & \cellcolor{green!15}$107.93\pm44.11$
        \\

        Neural-mid
        & $50\,048$
        & $21.49\pm8.95$
        & \cellcolor{green!15}$13.90\pm5.74$
        & \cellcolor{green!15}$108.62\pm44.82$
        \\

        Neural-large
        & $500\,096$
        & $72.16\pm30.68$
        & \cellcolor{green!15}$14.03\pm5.94$
        & \cellcolor{green!15}$109.59\pm46.42$
        \\

        \bottomrule
    \end{tabular}
\end{adjustbox}
\vspace{-0.2cm}
\caption{\textbf{Sensitivity to forward complexity.}
    One-time inversion and downstream optimization time for increasingly complex radial backbones with fixed $Q=128$. $\Rightarrow$ The backbone complexity only affects scalar inversion cost, while optimization time remains unchanged.
}
\vspace{-0.2cm}
\label{tab:backbone_ablation}
\end{wraptable}

$\bullet$~\emph{Sensitivity to forward-model complexity.}
Finally, we examine whether downstream optimization cost depends on the complexity of the forward radial model. Holding $Q=128$ fixed, we increase radial backbone size from $1.2$K to $500$K parameters (Table~\ref{tab:backbone_ablation}). More expressive backbones increase only slightly the one-time scalar inversion cost, while subsequent optimization remains constant at $14$\,ms/query, with less than $3\%$ variation between the smallest to largest model. Hence, once the preimage has been constructed, downstream optimization cost is independent  backbone complexity.

\textbf{\underline{Real-world application (AC power-flow):}} Finally, we test whether the same advantage of \method carry over to real-world application. For this, we use the IEEE 30-bus system from \texttt{pandapower} \citep{thurner2018pandapower} and maximize total renewable power injection across five buses, subject to the learned surrogate of the grid-security constraint. Each returned solution is then checked against the true AC power-flow. This yields a linear objective problem in which both efficient optimization over the learned model and feasibility under the real-world dynamics are important (details in Supplement~\ref{app:dgp_data}).

\begin{wraptable}{r}{0.6\textwidth}
\vspace{-0.55cm}
\setlength{\intextsep}{0pt}
\setlength{\columnsep}{1em}
\centering
\begin{adjustbox}{width=\linewidth}
\tiny
    \begin{tabular}{lcccrr}
        \toprule
        Method
        & Returned MW
        & All AC-feasible
        & Global optimality
        & Solver time $\downarrow$ & Rel. time increase
        \\
        \midrule

        ICNN + CLARABEL
        & \cellcolor{green!15}$109.92\pm0.21$
        & \cellcolor{red!15}$\;\;\text{\xmark}^{*}$
        & \cellcolor{red!15}$\;\;\text{\xmark}^{*}$
        & \cellcolor{red!15}$0.301\pm0.017$ 
        & \cellcolor{red!15} \textbf{$43\times$}
        \\
        
        Smooth MLP + SCIP
        & \cellcolor{green!15}$111.95\pm0.54$
        & \cellcolor{green!15}$\text{\cmark}$
        & \cellcolor{red!15}$\text{\xmark}$
        & \cellcolor{red!15}$26.985\pm8.380$ 
        & \cellcolor{red!15} \textbf{$3\;855\times$}
        \\

        Smooth MLP + IPOPT
        & \cellcolor{green!15}$111.95\pm0.54$
        & \cellcolor{green!15}$\text{\cmark}$
        & \cellcolor{red!15}$\text{\xmark}$
        & \cellcolor{red!15}$0.221\pm0.066$ 
        & \cellcolor{red!15}\textbf{$32\times$}
        \\
        CPWL MLP + Gurobi
        & \cellcolor{red!15}$113.51\pm0.43$
        & \cellcolor{red!15}$\text{\xmark}$
        & \cellcolor{green!15}$\text{\cmark}$
        & \cellcolor{red!15}$11.468\pm7.388$ 
        & \cellcolor{red!15} \textbf{$1\;638\times$}
        \\

        \midrule

        \textbf{\method}~(ours)
        & \cellcolor{green!15}${109.54\pm0.92}$
        & \cellcolor{green!15}$\text{\cmark}$
        & \cellcolor{green!15}$\text{\cmark}$
        & \cellcolor{green!15}$\bm{0.007\pm0.001}$ 
        & \cellcolor{green!15}\textbf{$1\times$}
        \\

        \bottomrule
        \multicolumn{6}{l}{\cellcolor{yellow!15}$*$: {ICNN had solver failures in $2$ runs, and only $6$ of $8$ remaining runs had clean optimal status.}}
    \end{tabular}
    \end{adjustbox}

    \vspace{-0.2cm}
    \caption{
        \textbf{Downstream optimization using the AC power-flow data.}
        We report the returned objective value, whether all runs are
        AC-feasible, global optimality with respect to the learned model,
        and mean solve time. Runtime increase is measured relative to \method.
        $\Rightarrow$ \method is the only approach that is both AC-feasible and globally
        optimal across all runs, while solving the problem in milliseconds.
    }
    \label{tab:real_simulator_optimization}
    \vspace{-0.2cm}
\end{wraptable}

Table~\ref{tab:real_simulator_optimization} reports the results across three criteria relevant in practice: the achieved objective value, feasibility under the AC power-flow simulator, and whether global optimality is established for the optimization problem defined by the learned predictor. \method is the only approach for which all runs are both AC-feasible and globally optimal with respect to the learned model, while requiring only $7$\,ms per solve. The baselines fail along at least one criterion: Gurobi solves the neural CPWL  optimization problem globally, but several returned solutions violate the AC constraint; SCIP (solver time-out) and IPOPT (local solver) return AC-feasible solutions but without global optimality across all runs; and the ICNN fails in both. At the same time, \method is $32\times$--$3\,855\times$ faster than optimization through the neural baselines. Overall, \method achieves the desired behavior for application in practice: it offers fast global optimization while producing solutions that remain feasible under the underlying AC simulator.

\textbf{\underline{Conclusion:}} To our knowledge, \method is the first framework for learning (i)~expressive forward predictors with (ii)~exact and explicit preimage representations by construction, and (iii) tractable global downstream optimization. We see broad applications in engineering and science, where inverse reasoning is central, but also in ML model inspection and auditing, where explicit preimages can help identify decision boundaries and safety-relevant regions of the input space.

\clearpage



\clearpage

\bibliography{bibliography}
\bibliographystyle{iclr2027_conference}

\clearpage

\appendix

\section{Additional theoretical results}

\subsection{Armotized factorization}\label{sec:amorized_factorization}

\begin{proposition}[Amortized factorization]
\label{prop:amortized_complexity}
Let $C_{\mathrm{fact}}(g)$ denote the one-time cost of instantiating the exact
geometric representation $\mathcal{P}_{F_\theta}(g)
=
\bigcup_{r\in\mathcal{A}(g)}
E_r(g)$ including the computation of the radii $R_r(g)=\phi_r^{-1}(g-\beta_r)$. Suppose that $K$ downstream queries are subsequently performed over $\mathcal{P}_{F_\theta}(g)$, and let $C_k^{\mathrm{geo}}$ denote the cost of the $k$-th query using only the resulting geometric representation. Then, the total cost is
\begin{align}
C_{\mathrm{fact}}(g)
+
\sum_{k=1}^{K}
C_k^{\mathrm{geo}}
.
\label{eq:amortized_complexity}
\end{align}
All dependence on the radial transformations $\phi_r$ is confined to the
one-time factorization cost $C_{\mathrm{fact}}(g)$. Hence, the
forward-model-dependent overhead per downstream query is
\begin{align}
\frac{C_{\mathrm{fact}}(g)}{K},
\end{align}
which vanishes as $K\rightarrow\infty$.
\end{proposition}

\begin{proof}
    See Supplement~\ref{sec:proofs}.
\end{proof}

Proposition~\ref{prop:amortized_complexity} makes the computational effect of factorization explicit. For each target level, \method pays the cost of the expressive radial transformations only once when constructing the preimage; every subsequent membership query, projection, or downstream optimization problem reuses the same finite geometry without evaluating $\phi_r$ again. Consequently, the forward-model-dependent cost is amortized across repeated uses of the learned preimage and becomes negligible when the same target geometry supports many downstream queries.

\clearpage

\subsection{Approximation rates}\label{sec:approximation_rate}

\begin{theorem}[Finite-expert approximation rate]
\label{thm:approximation_rate}
Let $\mathcal{X}\subset\mathbb{R}^d$ be compact and let
$f:\mathcal{X}\rightarrow\mathbb{R}$ be $L$-Lipschitz with $L>0$.
Suppose that, for some $h>0$, there exist centers
$c_1,\ldots,c_Q\in\mathcal{X}$ such that
\begin{align}
\sup_{x\in\mathcal{X}}
\min_{r=1,\ldots,Q}
\|x-c_r\|_2
\leq
h.
\label{eq:covering_radius}
\end{align}
Then there exists a \method predictor from the restricted isotropic quadratic subclass,
\begin{align}
F_\theta(x)
=
\min_{r=1,\ldots,Q}
\left[
\beta_r+\lambda\|x-c_r\|_2^2
\right],
\end{align}
with the common coefficient $\lambda=\frac{L}{2h}$, such that
\begin{align}
\sup_{x\in\mathcal{X}}
\lvert F_\theta(x)-f(x)\rvert
\leq
Lh.
\label{eq:finite_expert_rate}
\end{align}
Consequently, the corresponding preimages satisfy, simultaneously for every target level $g\in\mathbb{R}$,
\begin{align}
\mathcal{P}_f(g-Lh)
\subseteq
\mathcal{P}_{F_\theta}(g)
\subseteq
\mathcal{P}_f(g+Lh).
\label{eq:finite_expert_preimage_rate}
\end{align}
\end{theorem}
\begin{proof}
    See Supplement~\ref{sec:proofs}.
\end{proof}

Theorem~\ref{thm:approximation_rate} makes the approximation--complexity tradeoff explicit: if the $Q$ expert centers cover the domain at resolution $h$, then an isotropic quadratic \method predictor achieves uniform error at most $Lh$. At the same time, the same resolution controls the learned preimages uniformly across all target levels through a corresponding target-space error of size $Lh$.

\begin{corollary}[Rate in the number of experts]
\label{cor:expert_rate}
Let $h_Q(\mathcal{X})$ denote the optimal covering radius of $\mathcal{X}$ using $Q$ centers. Under the assumptions of Theorem~\ref{thm:approximation_rate},
\begin{align}
\inf_{F_\theta\in\mathcal{F}_Q}
\sup_{x\in\mathcal{X}}
\lvert F_\theta(x)-f(x)\rvert
\leq
L h_Q(\mathcal{X}),
\end{align}
where $\mathcal{F}_Q$ denotes the restricted isotropic quadratic \method class with at most $Q$ experts.

In particular, if the covering numbers of $\mathcal{X}$ satisfy
\begin{align}
N(\mathcal{X},h)
\leq
C_{\mathcal{X}}h^{-d},
\end{align}
then
\begin{align}
\inf_{F_\theta\in\mathcal{F}_Q}
\sup_{x\in\mathcal{X}}
\lvert F_\theta(x)-f(x)\rvert
\leq
L
\left(
\frac{C_{\mathcal{X}}}{Q}
\right)^{1/d}.
\label{eq:q_approximation_rate}
\end{align}
Hence, on bounded $d$-dimensional domains, the restricted \method subclass achieves the constructive rate
\begin{align}
\|F_\theta-f\|_\infty
=
\mathcal O\left(Q^{-1/d}\right).
\end{align}
\end{corollary}
\begin{proof}
    See Supplement~\ref{sec:proofs}.
\end{proof}
Corollary~\ref{cor:expert_rate} translates the geometric covering resolution into a rate in the number of experts. On bounded $d$-dimensional domains with standard covering-number scaling, the restricted \method class achieves the constructive rate $\|F_\theta-f\|_\infty=\mathcal O(Q^{-1/d})$.

\clearpage

\subsection{From forward errors to optimization guarantees}

The optimization guarantees in Section~\ref{sec:optimization} are exact with respect to the learned predictor $F_\theta$. A uniform approximation bound provides a direct link from these guarantees to the unknown ground-truth problem. In particular, if
\begin{align}
\sup_{x\in\mathcal{X}}
\lvert F_\theta(x)-f(x)\rvert
<
\varepsilon,
\end{align}
then shifting the target level by $\varepsilon$ yields the enclosure
\begin{align}
\mathcal{P}_{F_\theta}(g-\varepsilon)
\subseteq
\mathcal{P}_f(g)
\subseteq
\mathcal{P}_{F_\theta}(g+\varepsilon).
\label{eq:certified_preimage_sandwich}
\end{align}

The learned preimage at the tightened level $g-\varepsilon$ is guaranteed to contain only ground-truth feasible points. Conversely, every ground-truth feasible point is guaranteed to lie inside the learned preimage at the relaxed level $g+\varepsilon$.
Since both sets are explicit under \method, this enclosure can be propagated directly to downstream optimization, which yields a ground-truth feasible solution together with a computable global guarantee.

\begin{theorem}[Ground-truth optimization]
\label{thm:certified_optimization}
Suppose that
$\sup_{x\in\mathcal{X}}
\lvert F_\theta(x)-f(x)\rvert
\leq
\varepsilon$,
and let
$x_{\mathrm{in}}^\star
\in
\arg\min_{x\in\mathcal{P}_{F_\theta}(g-\varepsilon)}
J(x)$.
Assume that the relevant optimal values are finite. Then
\begin{align}
x_{\mathrm{in}}^\star
\in
\mathcal{P}_f(g)
\end{align}
and its suboptimality with respect to the ground-truth constrained problem satisfies
\begin{align}
0
\leq
J(x_{\mathrm{in}}^\star)
-
\inf_{x\in\mathcal{P}_f(g)}J(x)
\leq
\inf_{x\in\mathcal{P}_{F_\theta}(g-\varepsilon)}J(x)
-
\inf_{x\in\mathcal{P}_{F_\theta}(g+\varepsilon)}J(x).
\label{eq:certified_optimality_gap}
\end{align}
\end{theorem}

\begin{proof}
See Supplement~\ref{sec:proofs}.
\end{proof}

If a bound on the forward error is available, Theorem~\ref{thm:certified_optimization} gives an end-to-end guarantee. Optimizing over the tightened learned preimage produces a point that is guaranteed to satisfy the ground-truth constraint. Moreover, solving the same problem at the relaxed target level yields a computable certificate on how far this solution can be from the unknown ground-truth optimum. As both bounding preimages are tractable finite unions under \method, the certificate is obtained using the same regionwise optimization procedure as in \Eqref{eq:regionwise_optimization}.

\clearpage

\subsection{A Galois perspective}\label{sec:galois}

\textbf{Galois perspective:}
For any scalar predictor $F_\theta$, let
\begin{align}
\mathcal P_{F_\theta}^\star(A)
\coloneqq
\sup_{x\in A}F_\theta(x),
\end{align}
with $\mathcal P_{F_\theta}^\star(\varnothing)=-\infty$. Then, $\mathcal P_{F_\theta}$ and $\mathcal P_{F_\theta}^\star$ satisfy
\begin{align}
\mathcal P_{F_{\theta}}^\star(A)\leq g
\quad\Longleftrightarrow\quad
A\subseteq\mathcal P_{F_{\theta}}(g),
\end{align}
and therefore form a Galois connection
$\mathcal P_{F_\theta}^\star\dashv\mathcal P_{F_\theta}$ between
$(2^\mathcal{X},\subseteq)$ and $(\overline{\mathbb{R}},\leq)$~\citep{Cousot.1977,Cousot.1979}.
This relation is canonical to any scalar predictor; the idea of
\method is to construct predictors for which the corresponding backward map
$\mathcal P_{F_\theta}(g)$ admits an exact and explicit finite representation.

\subsection{Helping lemmas}\label{sec:lemmas}
\begin{lemma}[Uniform approximation by isotropic quadratic experts] \label{lem:quadratic_universal} Let $\mathcal{X}\subset\mathbb{R}^d$ be compact and let $f\in C(\mathcal{X})$. Then, for every $\varepsilon>0$, there exist a finite number of experts $Q$, centers $c_1,\ldots,c_Q\in\mathcal{X}$, offsets $\beta_1,\ldots,\beta_Q\in\mathbb{R}$, and a common coefficient $\lambda>0$ such that 
\begin{align} 
F_\theta(x) = \min_{r=1,\ldots,Q} \left[ \beta_r+\lambda\|x-c_r\|_2^2 \right] 
\end{align} 
satisfies 
\begin{align} 
\sup \{ x\in\mathcal{X} : \; \lvert F_\theta(x)-f(x)\rvert < \varepsilon\}. 
\end{align} 
Hence, the restricted isotropic quadratic subclass of \method is dense in $C(\mathcal{X})$ under the uniform norm. 
\end{lemma}
\begin{proof}
    See Supplement~\ref{sec:proofs}.
\end{proof}

\clearpage
\section{Proofs}\label{sec:proofs}



\begin{proof}[Proof of Theorem~\ref{thm:exact_prefact}]
Consider first a finite target level $g\in\mathbb{R}$. By definition of the learned preimage,
\begin{align}
x\in\mathcal P_{F_{\theta}}(g)
\quad\Longleftrightarrow\quad
x\in\mathcal{X}
\ \text{and}
F_\theta(x)\leq g.
\end{align}
Using the minimum representation of $F_\theta$ in~\Eqref{eq:prefact}, we have
\begin{align}
F_\theta(x)\leq g
&\quad\Longleftrightarrow\quad
\min_{r=1,\ldots,Q}
\left[
\beta_r+\phi_r(d_r(x))
\right]
\leq g \\
&\quad\Longleftrightarrow\quad
\exists r\in\{1,\ldots,Q\}:
\beta_r+\phi_r(d_r(x))
\leq g.
\label{eq:proof_min_union}
\end{align}
Thus, the outer minimum in the forward predictor becomes a disjunction over experts under the target-level query.

Because $d_r(x)\geq0$, $\phi_r(0)=0$, and $\phi_r$ is increasing,
\begin{align}
\phi_r(d_r(x))\geq0.
\end{align}
Consequently, an expert with $g<\beta_r$ cannot satisfy the inequality in~\Eqref{eq:proof_min_union}. We may therefore restrict attention to the active set
\begin{align}
\mathcal{A}(g)
=
\{r:g\geq\beta_r\}.
\end{align}
For every $r\in\mathcal{A}(g)$,
\begin{align}
\beta_r+\phi_r(d_r(x))\leq g
&\quad\Longleftrightarrow\quad
\phi_r(d_r(x))\leq g-\beta_r
\\
&\quad\Longleftrightarrow\quad
d_r(x)\leq\phi_r^{-1}(g-\beta_r)
\\
&\quad\Longleftrightarrow\quad
d_r(x)\leq R_r(g),
\end{align}
where the second equivalence follows from strict monotonicity of $\phi_r$.

Substituting the definition
\begin{align}
d_r(x)
=
\sqrt{(x-c_r)^\top A_r(x-c_r)}
\end{align}
and using $R_r(g)\geq0$ gives
\begin{align}
d_r(x)\leq R_r(g)
\quad\Longleftrightarrow\quad
(x-c_r)^\top A_r(x-c_r)
\leq R_r(g)^2.
\end{align}
The latter condition is precisely $x\in E_r(g)$. Combining the preceding equivalences therefore yields
\begin{align}
x\in\mathcal P_{F_{\theta}}(g)
&\quad\Longleftrightarrow\quad
x\in E_r(g)
\quad\text{for at least one }r\in\mathcal{A}(g)
\\
&\quad\Longleftrightarrow\quad
x\in
\bigcup_{r\in\mathcal{A}(g)}
E_r(g).
\end{align}
Since this holds for every $x\in\mathcal{X}$,
\begin{align}
\mathcal P_{F_{\theta}}(g)
=
\bigcup_{r\in\mathcal{A}(g)}
E_r(g).
\end{align}

For the extended target levels, the result follows by the natural conventions. Since $F_\theta$ is real-valued,
\begin{align}
\mathcal P_{F_{\theta}}(-\infty)=\varnothing,
\qquad
\mathcal P_{F_{\theta}}(+\infty)=\mathcal{X}.
\end{align}
At $g=-\infty$, the active set is empty and hence the union is empty. At $g=+\infty$, all experts are active; defining $R_r(+\infty)=+\infty$ makes every $E_r(+\infty)=\mathcal{X}$, so their union equals $\mathcal{X}$. Thus the stated factorization holds for all $g\in\overline{\mathbb{R}}$.
\end{proof}

\clearpage

\begin{proof}[Proof of \Eqref{eq:regionwise_optimization}]
For each active region, define
\begin{align}
v_r
\coloneqq
\inf_{x\in E_r(g)} J(x).
\end{align}

Since
\begin{align}
E_r(g)
\subseteq
\mathcal P_{F_{\theta}}(g)
\qquad
\forall r\in\mathcal{A}(g),
\end{align}
optimization over the larger set cannot yield a larger infimum. Therefore,
\begin{align}
\inf_{x\in\mathcal P_{F_{\theta}}(g)} J(x)
\leq
v_r
\qquad
\forall r\in\mathcal{A}(g),
\end{align}
and hence
\begin{align}
\inf_{x\in\mathcal P_{F_{\theta}}(g)} J(x)
\leq
\min_{r\in\mathcal{A}(g)} v_r.
\label{eq:regionwise_proof_upper}
\end{align}

Conversely, since
\begin{align}
\mathcal P_{F_{\theta}}(g)
=
\bigcup_{r\in\mathcal{A}(g)} E_r(g),
\end{align}
every $x\in\mathcal P_{F_{\theta}}(g)$ belongs to at least one active region. Thus, for some
$r\in\mathcal{A}(g)$,
\begin{align}
x\in E_r(g).
\end{align}
By definition of $v_r$,
\begin{align}
J(x)
\geq
v_r
\geq
\min_{s\in\mathcal{A}(g)} v_s.
\end{align}
Since this holds for every $x\in\mathcal P_{F_{\theta}}(g)$, taking the infimum over
$\mathcal P_{F_{\theta}}(g)$ gives
\begin{align}
\inf_{x\in\mathcal P_{F_{\theta}}(g)} J(x)
\geq
\min_{r\in\mathcal{A}(g)} v_r.
\label{eq:regionwise_proof_lower}
\end{align}

Combining~\Eqref{eq:regionwise_proof_upper} and~\Eqref{eq:regionwise_proof_lower},
we obtain
\begin{align}
\inf_{x\in\mathcal P_{F_{\theta}}(g)} J(x)
=
\min_{r\in\mathcal{A}(g)}
\inf_{x\in E_r(g)} J(x).
\end{align}

Finally, suppose that the regionwise infima are attained. Since
$\mathcal{A}(g)$ is finite, there exists
\begin{align}
r^\star
\in
\arg\min_{r\in\mathcal{A}(g)} v_r.
\end{align}
Let
\begin{align}
x_{r^\star}^\star
\in
\arg\min_{x\in E_{r^\star}(g)} J(x).
\end{align}
Then
\begin{align}
J(x_{r^\star}^\star)
=
v_{r^\star}
=
\inf_{x\in\mathcal P_{F_{\theta}}(g)} J(x),
\end{align}
so $x_{r^\star}^\star$ is a globally optimal solution over the complete learned preimage.

\end{proof}

\clearpage

\begin{proof}[Proof of Theorem~\ref{thm:linear_optimization}]
The optimization of a linear functional over an ellipsoid is a standard
support-function calculation \citep{rockafellar1970convex,boyd2004convex}. We provide the short derivation here for completeness and then apply it to the finite union induced by \method.

Fix a target level $g$ with $\mathcal{A}(g)\neq\varnothing$. By Theorem~\ref{thm:exact_prefact},
\begin{align}
\mathcal{P}_{F_\theta}(g)
=
\bigcup_{r\in\mathcal{A}(g)}
E_r(g),
\end{align}
where
\begin{align}
E_r(g)
=
\left\{
x\in\mathbb{R}^d:
(x-c_r)^\top A_r(x-c_r)
\leq
R_r(g)^2
\right\}.
\end{align}
Since the union is finite,
\begin{align}
\inf_{x\in\mathcal{P}_{F_\theta}(g)}
a^\top x
=
\min_{r\in\mathcal{A}(g)}
\inf_{x\in E_r(g)}
a^\top x.
\label{eq:linear_opt_union}
\end{align}

We therefore solve the optimization problem over an arbitrary active region
$E_r(g)$. Since $A_r\succ0$, let $A_r^{1/2}$ denote its unique symmetric
positive-definite square root. Introduce the change of variables
\begin{align}
y
\coloneqq
A_r^{1/2}(x-c_r).
\end{align}
Equivalently,
\begin{align}
x
=
c_r+A_r^{-1/2}y.
\end{align}
Under this transformation, the ellipsoidal constraint becomes
\begin{align}
\|y\|_2
\leq
R_r(g),
\end{align}
while the objective satisfies
\begin{align}
a^\top x
=
a^\top c_r
+
a^\top A_r^{-1/2}y
=
a^\top c_r
+
\left(A_r^{-1/2}a\right)^\top y.
\end{align}

Hence,
\begin{align}
\inf_{x\in E_r(g)}
a^\top x
=
a^\top c_r
+
\inf_{\|y\|_2\leq R_r(g)}
\left(A_r^{-1/2}a\right)^\top y.
\label{eq:linear_opt_ball}
\end{align}
By the Cauchy--Schwarz inequality,
\begin{align}
\left(A_r^{-1/2}a\right)^\top y
\geq
-
\left\|A_r^{-1/2}a\right\|_2
\|y\|_2
\geq
-
R_r(g)
\left\|A_r^{-1/2}a\right\|_2.
\end{align}
This bound is attained at
\begin{align}
y_r^\star
=
-
R_r(g)
\frac{A_r^{-1/2}a}
{\left\|A_r^{-1/2}a\right\|_2}.
\label{eq:linear_opt_y}
\end{align}
The denominator is strictly positive because $a\neq0$ and $A_r$ is positive
definite. Moreover,
\begin{align}
\left\|A_r^{-1/2}a\right\|_2^2
=
a^\top A_r^{-1}a.
\end{align}
Defining
\begin{align}
s_r(a)
\coloneqq
\sqrt{a^\top A_r^{-1}a},
\end{align}
we therefore obtain
\begin{align}
\inf_{x\in E_r(g)}
a^\top x
=
a^\top c_r
-
R_r(g)s_r(a).
\label{eq:linear_opt_region_value}
\end{align}

Substituting~\Eqref{eq:linear_opt_region_value} into
\Eqref{eq:linear_opt_union} yields
\begin{align}
\inf_{x\in\mathcal{P}_{F_\theta}(g)}
a^\top x
=
\min_{r\in\mathcal{A}(g)}
\left[
a^\top c_r
-
R_r(g)\sqrt{a^\top A_r^{-1}a}
\right].
\end{align}

It remains to recover a globally optimal point. Transforming
$y_r^\star$ in~\Eqref{eq:linear_opt_y} back to the original coordinates gives
\begin{align}
x_r^\star
&=
c_r+A_r^{-1/2}y_r^\star
\\
&=
c_r
-
R_r(g)
\frac{A_r^{-1}a}
{\sqrt{a^\top A_r^{-1}a}}.
\label{eq:linear_opt_region_optimizer}
\end{align}
Let
\begin{align}
r^\star
\in
\arg\min_{r\in\mathcal{A}(g)}
\left[
a^\top c_r
-
R_r(g)\sqrt{a^\top A_r^{-1}a}
\right].
\end{align}
Then $x_{r^\star}^\star\in E_{r^\star}(g)\subseteq
\mathcal{P}_{F_\theta}(g)$ and attains the smallest objective value over all
active regions. Consequently,
\begin{align}
x^\star
=
c_{r^\star}
-
\frac{R_{r^\star}(g)}
{\sqrt{a^\top A_{r^\star}^{-1}a}}
A_{r^\star}^{-1}a
\end{align}
is a globally optimal solution over the complete learned preimage.

Finally, substituting
\begin{align}
R_r(g)
=
\phi_r^{-1}(g-\beta_r)
\end{align}
into~\Eqref{eq:linear_opt_region_value} gives the target-dependent value function
\begin{align}
V_a(g)
=
\min_{r:\,g\geq\beta_r}
\left[
a^\top c_r
-
\sqrt{a^\top A_r^{-1}a}\,
\phi_r^{-1}(g-\beta_r)
\right].
\end{align}
Thus, the globally optimal value over the generally nonconvex learned preimage
is obtained by evaluating one closed-form scalar expression for each active
expert and selecting the minimum.
\end{proof}

\clearpage

\begin{proof}[Proof of Theorem~\ref{thm:universal}]
Let $\varepsilon>0$. By Lemma~\ref{lem:quadratic_universal}, there exists a finite predictor
$G_\theta$ from the restricted isotropic quadratic subclass,
\begin{align}
G_\theta(x)
=
\min_{r=1,\ldots,Q}
\left[
\beta_r+\lambda\|x-c_r\|_2^2
\right],
\end{align}
with a common coefficient $\lambda>0$, such that
\begin{align}
\sup_{x\in\mathcal{X}}
\lvert G_\theta(x)-f(x)\rvert
<
\frac{\varepsilon}{2}.
\label{eq:universal_auxiliary_approximation}
\end{align}

Define the shifted predictor
\begin{align}
F_\theta(x)
\coloneqq
G_\theta(x)-\frac{\varepsilon}{2}.
\label{eq:universal_shifted_predictor}
\end{align}
This remains in the same restricted isotropic quadratic subclass, since
\begin{align}
F_\theta(x)
=
\min_{r=1,\ldots,Q}
\left[
\left(\beta_r-\frac{\varepsilon}{2}\right)
+
\lambda\|x-c_r\|_2^2
\right].
\end{align}

By~\Eqref{eq:universal_auxiliary_approximation}, for every
$x\in\mathcal{X}$,
\begin{align}
f(x)-\frac{\varepsilon}{2}
<
G_\theta(x)
<
f(x)+\frac{\varepsilon}{2}.
\end{align}
Subtracting $\varepsilon/2$ gives
\begin{align}
f(x)-\varepsilon
<
F_\theta(x)
<
f(x)
\qquad
\forall x\in\mathcal{X}.
\label{eq:universal_one_sided}
\end{align}
Hence,
\begin{align}
\sup_{x\in\mathcal{X}}
\lvert F_\theta(x)-f(x)\rvert
<
\varepsilon.
\label{eq:universal_forward_bound}
\end{align}

It remains to compare the complete families of sublevel-set preimages.
Fix an arbitrary target level $g\in\mathbb{R}$.

First, let
\begin{align}
x\in\mathcal{P}_f(g).
\end{align}
Then $f(x)\leq g$, and~\Eqref{eq:universal_one_sided} gives
\begin{align}
F_\theta(x)
<
f(x)
\leq
g.
\end{align}
Therefore,
\begin{align}
x\in\mathcal{P}_{F_\theta}(g),
\end{align}
and thus
\begin{align}
\mathcal{P}_f(g)
\subseteq
\mathcal{P}_{F_\theta}(g).
\label{eq:universal_inner_preimage}
\end{align}

Conversely, let
\begin{align}
x\in\mathcal{P}_{F_\theta}(g).
\end{align}
Then $F_\theta(x)\leq g$. Again using~\Eqref{eq:universal_one_sided},
\begin{align}
f(x)
<
F_\theta(x)+\varepsilon
\leq
g+\varepsilon.
\end{align}
Hence,
\begin{align}
x\in\mathcal{P}_f(g+\varepsilon),
\end{align}
which yields
\begin{align}
\mathcal{P}_{F_\theta}(g)
\subseteq
\mathcal{P}_f(g+\varepsilon).
\label{eq:universal_outer_preimage}
\end{align}

Combining~\Eqref{eq:universal_inner_preimage} and
\Eqref{eq:universal_outer_preimage}, we obtain
\begin{align}
\mathcal{P}_f(g)
\subseteq
\mathcal{P}_{F_\theta}(g)
\subseteq
\mathcal{P}_f(g+\varepsilon).
\end{align}
Since $g\in\mathbb{R}$ was arbitrary, this inclusion holds simultaneously for every target level:
\begin{align}
\mathcal{P}_f(g)
\subseteq
\mathcal{P}_{F_\theta}(g)
\subseteq
\mathcal{P}_f(g+\varepsilon)
\qquad
\forall g\in\mathbb{R}.
\end{align}

Together with~\Eqref{eq:universal_forward_bound}, this proves that a single finite
\method predictor can approximate $f$ uniformly while simultaneously approximating its entire family of sublevel-set preimages up to an arbitrarily small shift in target level. Finally, the construction uses only the restricted subclass with
$A_r=I$ and $\phi_r(d)=\lambda d^2$, with the same common $\lambda>0$ for all experts.
\end{proof}

\clearpage

\begin{proof}[Proof of Theorem~\ref{thm:approximation_rate}]
Let $c_1,\ldots,c_Q\in\mathcal{X}$ satisfy
\begin{align}
\sup_{x\in\mathcal{X}}
\min_{r=1,\ldots,Q}
\|x-c_r\|_2
\leq
h.
\end{align}
Choose
\begin{align}
\lambda
\coloneqq
\frac{L}{2h},
\qquad
\beta_r
\coloneqq
f(c_r)-\frac{Lh}{2},
\end{align}
and define
\begin{align}
F_\theta(x)
=
\min_{r=1,\ldots,Q}
\left[
f(c_r)-\frac{Lh}{2}
+
\frac{L}{2h}\|x-c_r\|_2^2
\right].
\label{eq:rate_constructed_predictor}
\end{align}
This is a restricted \method predictor with
\begin{align}
A_r=I,
\qquad
\phi_r(d)=\frac{L}{2h}d^2
\end{align}
for every expert $r$.

We first establish the upper approximation bound. Fix any
$x\in\mathcal{X}$. By the covering assumption, there exists an expert
$r$ such that
\begin{align}
d
\coloneqq
\|x-c_r\|_2
\leq
h.
\end{align}
Since $f$ is $L$-Lipschitz,
\begin{align}
f(c_r)
\leq
f(x)+Ld.
\end{align}
Using expert $r$ as a candidate in the minimum,
\begin{align}
F_\theta(x)-f(x)
&\leq
Ld-\frac{Lh}{2}
+
\frac{L}{2h}d^2.
\end{align}
The right-hand side is increasing in $d\geq0$. Since $d\leq h$,
\begin{align}
F_\theta(x)-f(x)
&\leq
Lh-\frac{Lh}{2}
+
\frac{Lh}{2}
\\
&=
Lh.
\label{eq:rate_upper_bound}
\end{align}

We next establish the lower bound. Consider any expert
$r\in\{1,\ldots,Q\}$ and write
\begin{align}
d_r
\coloneqq
\|x-c_r\|_2.
\end{align}
By Lipschitz continuity,
\begin{align}
f(c_r)
\geq
f(x)-Ld_r.
\end{align}
Therefore,
\begin{align}
f(c_r)-\frac{Lh}{2}
+
\frac{L}{2h}d_r^2
-f(x)
&\geq
-Ld_r-\frac{Lh}{2}
+
\frac{L}{2h}d_r^2
\\
&=
\frac{L}{2h}(d_r-h)^2-Lh
\\
&\geq
-Lh.
\end{align}
Hence every expert satisfies
\begin{align}
f(c_r)-\frac{Lh}{2}
+
\frac{L}{2h}\|x-c_r\|_2^2
\geq
f(x)-Lh.
\end{align}
Taking the minimum over all experts preserves the bound:
\begin{align}
F_\theta(x)
\geq
f(x)-Lh.
\label{eq:rate_lower_bound}
\end{align}

Combining~\Eqref{eq:rate_upper_bound} and
\Eqref{eq:rate_lower_bound} yields
\begin{align}
\lvert F_\theta(x)-f(x)\rvert
\leq
Lh.
\end{align}
Since $x\in\mathcal{X}$ was arbitrary,
\begin{align}
\sup_{x\in\mathcal{X}}
\lvert F_\theta(x)-f(x)\rvert
\leq
Lh.
\end{align}

It remains to establish the corresponding preimage bounds. The uniform
approximation estimate implies
\begin{align}
f(x)-Lh
\leq
F_\theta(x)
\leq
f(x)+Lh
\qquad
\forall x\in\mathcal{X}.
\label{eq:rate_pointwise_bound}
\end{align}
Fix an arbitrary target level $g\in\mathbb{R}$.

If
\begin{align}
x\in\mathcal{P}_f(g-Lh),
\end{align}
then $f(x)\leq g-Lh$, and~\Eqref{eq:rate_pointwise_bound} gives
\begin{align}
F_\theta(x)
\leq
f(x)+Lh
\leq
g.
\end{align}
Thus,
\begin{align}
\mathcal{P}_f(g-Lh)
\subseteq
\mathcal{P}_{F_\theta}(g).
\end{align}

Conversely, if
\begin{align}
x\in\mathcal{P}_{F_\theta}(g),
\end{align}
then $F_\theta(x)\leq g$. Again using~\Eqref{eq:rate_pointwise_bound},
\begin{align}
f(x)
\leq
F_\theta(x)+Lh
\leq
g+Lh,
\end{align}
and therefore
\begin{align}
\mathcal{P}_{F_\theta}(g)
\subseteq
\mathcal{P}_f(g+Lh).
\end{align}

Since $g$ was arbitrary, we conclude that
\begin{align}
\mathcal{P}_f(g-Lh)
\subseteq
\mathcal{P}_{F_\theta}(g)
\subseteq
\mathcal{P}_f(g+Lh)
\qquad
\forall g\in\mathbb{R}.
\end{align}
This proves the theorem.
\end{proof}

\clearpage

\begin{proof}[Proof of Corollary~\ref{cor:expert_rate}]
Let
\begin{align}
h_Q(\mathcal{X})
\coloneqq
\inf_{c_1,\ldots,c_Q\in\mathcal{X}}
\sup_{x\in\mathcal{X}}
\min_{r=1,\ldots,Q}
\|x-c_r\|_2
\end{align}
denote the optimal covering radius of $\mathcal{X}$ using $Q$ centers.

Fix any $\eta>0$. By the definition of the infimum, there exist centers
$c_1,\ldots,c_Q\in\mathcal{X}$ such that
\begin{align}
\sup_{x\in\mathcal{X}}
\min_{r=1,\ldots,Q}
\|x-c_r\|_2
\leq
h_Q(\mathcal{X})+\eta.
\end{align}
Applying Theorem~\ref{thm:approximation_rate} with
\begin{align}
h
=
h_Q(\mathcal{X})+\eta
\end{align}
therefore yields a predictor
$F_\theta\in\mathcal{F}_Q$ satisfying
\begin{align}
\sup_{x\in\mathcal{X}}
\lvert F_\theta(x)-f(x)\rvert
\leq
L\left(h_Q(\mathcal{X})+\eta\right).
\end{align}
Hence,
\begin{align}
\inf_{F_\theta\in\mathcal{F}_Q}
\sup_{x\in\mathcal{X}}
\lvert F_\theta(x)-f(x)\rvert
\leq
L\left(h_Q(\mathcal{X})+\eta\right).
\end{align}
Since $\eta>0$ was arbitrary, letting $\eta\downarrow0$ gives
\begin{align}
\inf_{F_\theta\in\mathcal{F}_Q}
\sup_{x\in\mathcal{X}}
\lvert F_\theta(x)-f(x)\rvert
\leq
L h_Q(\mathcal{X}).
\label{eq:covering_radius_rate}
\end{align}

Now suppose that the covering numbers of $\mathcal{X}$ satisfy
\begin{align}
N(\mathcal{X},h)
\leq
C_{\mathcal{X}}h^{-d}.
\label{eq:covering_number_assumption}
\end{align}
Choose
\begin{align}
h
=
\left(
\frac{C_{\mathcal{X}}}{Q}
\right)^{1/d}.
\end{align}
Then~\Eqref{eq:covering_number_assumption} gives
\begin{align}
N(\mathcal{X},h)
\leq
C_{\mathcal{X}}
\left(
\frac{C_{\mathcal{X}}}{Q}
\right)^{-1}
=
Q.
\end{align}
Thus, $\mathcal{X}$ can be covered by at most $Q$ Euclidean balls of radius $h$, and therefore
\begin{align}
h_Q(\mathcal{X})
\leq
\left(
\frac{C_{\mathcal{X}}}{Q}
\right)^{1/d}.
\end{align}
Substituting this bound into~\Eqref{eq:covering_radius_rate} yields
\begin{align}
\inf_{F_\theta\in\mathcal{F}_Q}
\sup_{x\in\mathcal{X}}
\lvert F_\theta(x)-f(x)\rvert
\leq
L
\left(
\frac{C_{\mathcal{X}}}{Q}
\right)^{1/d}.
\end{align}
Consequently,
\begin{align}
\inf_{F_\theta\in\mathcal{F}_Q}
\|F_\theta-f\|_\infty
=
O\left(Q^{-1/d}\right),
\end{align}
which proves the corollary.
\end{proof}

\clearpage

\begin{proof}[Proof of Theorem~\ref{thm:certified_optimization}]
The uniform approximation bound implies that, for every $x\in\mathcal{X}$,
\begin{align}
f(x)-\varepsilon
\leq
F_\theta(x)
\leq
f(x)+\varepsilon.
\label{eq:certified_pointwise_bound}
\end{align}

We first establish the corresponding inclusion of feasible sets. Let
\begin{align}
x
\in
\mathcal{P}_{F_\theta}(g-\varepsilon).
\end{align}
Then
\begin{align}
F_\theta(x)
\leq
g-\varepsilon.
\end{align}
Using~\Eqref{eq:certified_pointwise_bound},
\begin{align}
f(x)
\leq
F_\theta(x)+\varepsilon
\leq
g,
\end{align}
and hence
\begin{align}
x
\in
\mathcal{P}_f(g).
\end{align}
Therefore,
\begin{align}
\mathcal{P}_{F_\theta}(g-\varepsilon)
\subseteq
\mathcal{P}_f(g).
\label{eq:certified_inner_inclusion}
\end{align}

Conversely, let
\begin{align}
x
\in
\mathcal{P}_f(g).
\end{align}
Then $f(x)\leq g$, and~\Eqref{eq:certified_pointwise_bound} gives
\begin{align}
F_\theta(x)
\leq
f(x)+\varepsilon
\leq
g+\varepsilon.
\end{align}
Thus,
\begin{align}
x
\in
\mathcal{P}_{F_\theta}(g+\varepsilon),
\end{align}
so
\begin{align}
\mathcal{P}_f(g)
\subseteq
\mathcal{P}_{F_\theta}(g+\varepsilon).
\label{eq:certified_outer_inclusion}
\end{align}

Combining~\Eqref{eq:certified_inner_inclusion} and
\Eqref{eq:certified_outer_inclusion} yields
\begin{align}
\mathcal{P}_{F_\theta}(g-\varepsilon)
\subseteq
\mathcal{P}_f(g)
\subseteq
\mathcal{P}_{F_\theta}(g+\varepsilon).
\label{eq:certified_sandwich_proof}
\end{align}

Since
\begin{align}
x_{\mathrm{in}}^\star
\in
\mathcal{P}_{F_\theta}(g-\varepsilon),
\end{align}
the first inclusion in~\Eqref{eq:certified_sandwich_proof} immediately gives
\begin{align}
x_{\mathrm{in}}^\star
\in
\mathcal{P}_f(g),
\end{align}
so $x_{\mathrm{in}}^\star$ is feasible for the ground-truth constrained problem.

Next, define
\begin{align}
v_{\mathrm{in}}
&\coloneqq
\inf_{x\in\mathcal{P}_{F_\theta}(g-\varepsilon)} J(x),
\\
v_f
&\coloneqq
\inf_{x\in\mathcal{P}_f(g)} J(x),
\\
v_{\mathrm{out}}
&\coloneqq
\inf_{x\in\mathcal{P}_{F_\theta}(g+\varepsilon)} J(x).
\end{align}
Because minimization over a larger feasible set cannot increase the optimal value, the inclusions in~\Eqref{eq:certified_sandwich_proof} imply
\begin{align}
v_{\mathrm{out}}
\leq
v_f
\leq
v_{\mathrm{in}}.
\label{eq:certified_value_order}
\end{align}

By definition of $x_{\mathrm{in}}^\star$,
\begin{align}
J(x_{\mathrm{in}}^\star)
=
v_{\mathrm{in}}.
\end{align}
Hence,
\begin{align}
J(x_{\mathrm{in}}^\star)-v_f
=
v_{\mathrm{in}}-v_f
\geq
0,
\end{align}
where the inequality follows from~\Eqref{eq:certified_value_order}. Moreover, since
$v_f\geq v_{\mathrm{out}}$,
\begin{align}
v_{\mathrm{in}}-v_f
\leq
v_{\mathrm{in}}-v_{\mathrm{out}}.
\end{align}
Therefore,
\begin{align}
0
\leq
J(x_{\mathrm{in}}^\star)
-
\inf_{x\in\mathcal{P}_f(g)} J(x)
\leq
\inf_{x\in\mathcal{P}_{F_\theta}(g-\varepsilon)} J(x)
-
\inf_{x\in\mathcal{P}_{F_\theta}(g+\varepsilon)} J(x),
\end{align}
which proves the stated suboptimality certificate.
\end{proof}

\clearpage

\begin{proof}[Proof of Lemma~\ref{lem:quadratic_universal}]
The use of finite minima of quadratic basis functions is related to
min-of-quadratics approximation constructions based on semiconcave and
semiconvex representations; see, e.g.,~\citet{dower2025coverage}. We give a
direct construction for the restricted \method subclass considered here.

Let $\varepsilon>0$ and set
\begin{align}
\eta
\coloneqq
\frac{\varepsilon}{3}.
\end{align}
Since $\mathcal{X}$ is compact and $f\in C(\mathcal{X})$, the function $f$ is
uniformly continuous. Hence, there exists $\rho>0$ such that
\begin{align}
\|x-y\|_2<\rho
\quad\Longrightarrow\quad
\lvert f(x)-f(y)\rvert<\eta
\qquad
\forall x,y\in\mathcal{X}.
\label{eq:uat_uniform_continuity}
\end{align}

Compactness also implies that $f$ attains its minimum and maximum on
$\mathcal{X}$. Define
\begin{align}
m
\coloneqq
\min_{x\in\mathcal{X}} f(x),
\qquad
M
\coloneqq
\max_{x\in\mathcal{X}} f(x),
\qquad
\Delta
\coloneqq
M-m.
\end{align}
Choose $\lambda>0$ such that
\begin{align}
\lambda\rho^2>\Delta.
\label{eq:uat_lambda}
\end{align}
Next, choose $\delta>0$ such that
\begin{align}
\delta<\rho,
\qquad
\lambda\delta^2<\eta.
\label{eq:uat_delta}
\end{align}

Since $\mathcal{X}$ is compact, there exists a finite $\delta$-net
\begin{align}
\{c_1,\ldots,c_Q\}
\subseteq
\mathcal{X}
\end{align}
such that, for every $x\in\mathcal{X}$, there exists at least one
$r\in\{1,\ldots,Q\}$ satisfying
\begin{align}
\|x-c_r\|_2<\delta.
\label{eq:uat_cover}
\end{align}
For each center, set
\begin{align}
\beta_r
\coloneqq
f(c_r),
\end{align}
and define
\begin{align}
F_\theta(x)
=
\min_{r=1,\ldots,Q}
\left[
f(c_r)+\lambda\|x-c_r\|_2^2
\right].
\label{eq:uat_constructed_predictor}
\end{align}
This is a special case of \method with
\begin{align}
A_r=I,
\qquad
\phi_r(d)=\lambda d^2
\end{align}
for every expert $r$.

We now bound the approximation error at an arbitrary
$x\in\mathcal{X}$.

\textbf{Upper bound.}
By~\Eqref{eq:uat_cover}, there exists a center $c_r$ such that
\begin{align}
\|x-c_r\|_2<\delta.
\end{align}
Since $\delta<\rho$, uniform continuity in
\Eqref{eq:uat_uniform_continuity} gives
\begin{align}
f(c_r)
<
f(x)+\eta.
\end{align}
Using this expert as a candidate in the minimum and applying
\Eqref{eq:uat_delta},
\begin{align}
F_\theta(x)
&\leq
f(c_r)+\lambda\|x-c_r\|_2^2
\\
&<
f(x)+\eta+\lambda\delta^2
\\
&<
f(x)+2\eta.
\label{eq:uat_upper}
\end{align}

\textbf{Lower bound.}
Consider any expert $r\in\{1,\ldots,Q\}$. We distinguish two cases.

If
\begin{align}
\|x-c_r\|_2<\rho,
\end{align}
then uniform continuity implies
\begin{align}
f(c_r)
>
f(x)-\eta.
\end{align}
Since the quadratic term is nonnegative,
\begin{align}
f(c_r)+\lambda\|x-c_r\|_2^2
>
f(x)-\eta.
\label{eq:uat_lower_near}
\end{align}

If instead
\begin{align}
\|x-c_r\|_2\geq\rho,
\end{align}
then, using $f(c_r)\geq m$ and~\Eqref{eq:uat_lambda},
\begin{align}
f(c_r)+\lambda\|x-c_r\|_2^2
&\geq
m+\lambda\rho^2
\\
&>
m+\Delta
\\
&=
M
\\
&\geq
f(x).
\label{eq:uat_lower_far}
\end{align}

Thus, in either case,
\begin{align}
f(c_r)+\lambda\|x-c_r\|_2^2
>
f(x)-\eta.
\end{align}
Since this holds for every expert, taking the minimum over
$r=1,\ldots,Q$ yields
\begin{align}
F_\theta(x)
>
f(x)-\eta.
\label{eq:uat_lower}
\end{align}

Combining~\Eqref{eq:uat_upper} and~\Eqref{eq:uat_lower},
\begin{align}
f(x)-\eta
<
F_\theta(x)
<
f(x)+2\eta.
\end{align}
Therefore,
\begin{align}
\lvert F_\theta(x)-f(x)\rvert
<
2\eta
=
\frac{2\varepsilon}{3}
<
\varepsilon.
\end{align}
Since $x\in\mathcal{X}$ was arbitrary,
\begin{align}
\sup_{x\in\mathcal{X}}
\lvert F_\theta(x)-f(x)\rvert
<
\varepsilon.
\end{align}
Hence, the restricted isotropic quadratic subclass is dense in
$C(\mathcal{X})$ under the uniform norm.
\end{proof}

\clearpage

\section{Details on the experiments}\label{sec:experiment_supplement}

\subsection{Data-generating processes and datasets}
\label{app:dgp_data}

We evaluate the methods on three complementary regression-to-preimage
problems. The first two benchmarks provide analytic ground truth and allow
dense evaluation of increasingly challenging nonconvex sublevel sets, whereas
the third replaces the analytic data-generating process by a nonlinear AC
power-flow simulator. Table~\ref{tab:dgp_data_summary} summarizes the domains,
data construction, and ground-truth evaluation used throughout the
experiments.

\begin{table}[h!]
    \centering
    \begin{adjustbox}{max width=\textwidth}
    \begin{tabular}{l|l|l}
        \toprule
        \textbf{Benchmark} & \textbf{Setting} & \textbf{Configuration} \\
        \midrule

        \multirow{7}{*}{Six-Hump Camel}
        & Input dimension / domain
        & $2$; $x\in[-2,2]\times[-1.5,1.5]$
        \\
        & Model coordinates
        & $u=(x_1/2,\,2x_2/3)\in[-1,1]^2$
        \\
        & Train / validation / test
        & $100\,000\,/\,5\,000\,/\,5\,000$
        \\
        & Preimage ground truth
        & Analytic $f_{\mathrm{SH}}$ evaluated on an independent
          $1201\times1201$ grid
        \\
        & Thresholds
        & $g\in\{-0.8,-0.1,0.4,0.9,2.15\}$
        \\
        \midrule

        \multirow{7}{*}{Complex powers}
        & Input dimension / domain
        & $2$; $\{x\in\mathbb{R}^2:\|x\|_2\leq1\}$
        \\
        & Powers
        & $p\in\{4,6,8,10,12\}$
        \\
        & Train / validation / test
        & $16\,384\,/\,4\,096\,/\,8\,192$ for each $(p,s)$
        \\
        & Sampling
        & Independent uniform-in-area samples from the unit disk
        \\
        & Preimage ground truth
        & Analytic $f_p$ on an independent $801\times801$ grid;
          $19$ training-label quantile thresholds
        \\
        \midrule

        \multirow{7}{*}{Pandapower}
        & Input dimension / domain
        & $5$ renewable injections; $x\in[0,25]^5$ MW
        \\
        & Model coordinates
        & $u=x/25\in[0,1]^5$
        \\
        & Physical system
        & IEEE-30 network; renewable buses $\{7,17,25,2,15\}$
        \\
        & Dataset
        & $20\,000$ deterministic scrambled Sobol designs
        \\
        & Train / validation / test
        & $12\,000\,/\,4\,000\,/\,4\,000$
        \\

        \bottomrule
    \end{tabular}
    \end{adjustbox}
    \caption{
        Data-generating processes and datasets used in the experiments.
        The synthetic benchmarks admit exact analytic labels and dense
        ground-truth preimage evaluation. For Pandapower, labels and physical
        feasibility are obtained from the frozen nonlinear AC power-flow
        simulation.
    }
    \label{tab:dgp_data_summary}
\end{table}

\paragraph{Six-Hump Camel.}
The Six-Hump Camel function provides a low-dimensional but strongly nonconvex benchmark with a multimodal landscape and nontrivial sublevel-set geometry. Its two-dimensional domain additionally permits dense visualization of the true preimage, which makes it useful for studying both the geometric capacity of \method and post-hoc preimage extraction.

For $x=(x_1,x_2)$, the regression target is
\begin{equation}
    f_{\mathrm{SH}}(x)
    =
    \left(
        4-2.1x_1^2+\frac{x_1^4}{3}
    \right)x_1^2
    +x_1x_2
    +\left(-4+4x_2^2\right)x_2^2 .
    \label{eq:six_hump_dgp}
\end{equation}
Supervised labels are obtained by evaluating Eq.~\eqref{eq:six_hump_dgp} directly at these inputs; hence the synthetic targets are noise-free. Inputs are affinely mapped from the physical rectangle $[-2,2]\times[-1.5,1.5]$ to $[-1,1]^2$ before training, while the target is left on its original scale.

Forward prediction is evaluated on the fixed $5\,000$-point test set. Separately, geometric quantities are computed on an independent $1201\times1201$ grid over the full physical rectangle. At threshold $g$, the ground-truth membership of every grid point is obtained directly as 
\begin{equation}
    \mathbf{1}\!\left\{f_{\mathrm{SH}}(x)\leq g\right\},
\end{equation}
rather than from sampled or learned labels. This grid is used for the reported ground-truth preimage IoU and for the independent auditing of extracted preimages.

\begin{figure}[h]
\vspace{-0.2cm}
\centering
\includegraphics[width=\textwidth, trim=0cm 0cm 0cm 0cm, clip]{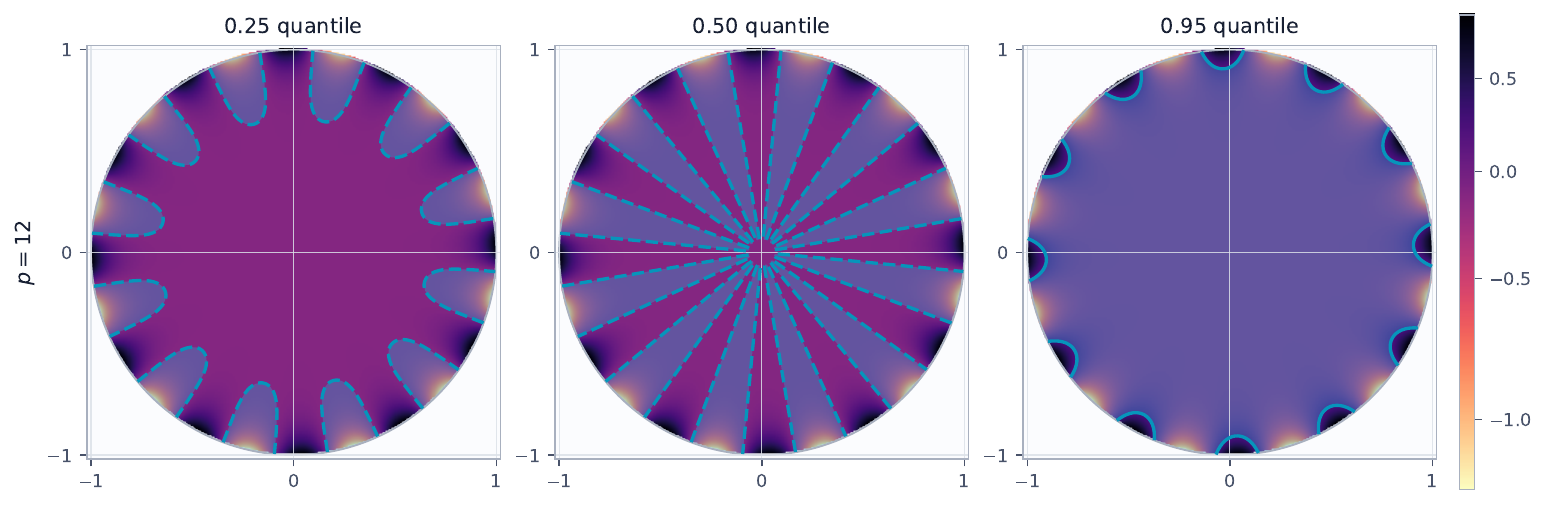}
\vspace{-0.7cm}
\caption{Complex power induces highly multimodel, disconnected preimages. Shown are the preimages at different target level quantiles. \method can approximate these arbitrarily close.}
\label{fig:complex_squaring}
\vspace{-0.4cm}
\end{figure}

\paragraph{Complex powers.}
The complex-power family is a controlled stress test in which the difficulty of the inverse geometry increases while the underlying data-generating process remains exactly known (see Figure~\ref{fig:complex_squaring}). Increasing the power $p$ produces progressively more oscillatory and multimodal sublevel sets, which allows forward approximation and preimage recovery to be studied separately under increasing geometric complexity.

Writing $z=x_1+\mathrm{i}x_2$, the raw response is
\begin{equation}
    f_p(x)
    =
    \operatorname{Re}\!\left(
        e^{-\mathrm{i}\alpha}z^p
    \right),
    \qquad
    \alpha=25^\circ,
    \qquad
    p\in\{4,6,8,10,12\}.
    \label{eq:complex_power_dgp}
\end{equation}
The supervised regression target contains the fixed offset
\begin{equation}
    y_p(x)=f_p(x)-\delta,
    \qquad
    \delta=0.13\cos(\alpha)\approx0.1178200.
    \label{eq:complex_power_offset}
\end{equation}
Accordingly, a threshold $g$ in model-output coordinates corresponds to
\begin{equation}
    g_{\mathrm{raw}}=g+\delta
\end{equation}
for the analytic DGP.

For every $p$, training, validation, and test inputs are sampled independently and uniformly in area from the unit disk. Their labels are obtained by direct evaluation of \Eqref{eq:complex_power_dgp} and \Eqref{eq:complex_power_offset}.

Forward RMSE is computed on the corresponding $8\,192$ held-out test points. Preimage accuracy is evaluated independently on an $801\times801$ Cartesian grid restricted to the unit disk. For each run, $19$ thresholds corresponding to the $5\%,10\%,\ldots,95\%$ quantiles of the training labels are considered, and true grid membership is obtained directly from the analytic condition 
\begin{equation}
    f_p(x)\leq g_{\mathrm{raw}}.
\end{equation}
Thus the backward/preimage evaluation does not reuse the supervised test sample.

\paragraph{Pandapower.}
The Pandapower benchmark \citep{thurner2018pandapower} moves from analytic functions to a physics-based power-system simulation. It represents a realistic application setting: the learned constraint approximates AC-grid security under five
simultaneous renewable-power injections, and optimized designs are subsequently checked again by the nonlinear AC power-flow model.

We use the IEEE-30 network and place five static renewable generators at buses $7$, $17$, $25$, $2$, and $15$, in this order, with zero reactive-power injection. For the physical design vector
\begin{equation}
    x=(P_1,\ldots,P_5)\in[0,25]^5\;\mathrm{MW},
\end{equation}
a Newton--Raphson AC power flow is solved using flat initialization and at most $20$ iterations. For every converged design, the supervised scalar target is the maximum normalized network-security loading,
\begin{equation}
    f_{\mathrm{AC}}(x)
    =
    \max\left\{
        \max_{\ell}
        \frac{\mathrm{loading}_{\ell}}{100},
        \;
        \max_{t}
        \frac{\mathrm{loading}_{t}}{100},
        \;
        \max_{b}
        \frac{V_b}{V_b^{\max}},
        \;
        \max_{b}
        \frac{V_b^{\min}}{V_b}
    \right\},
    \label{eq:pandapower_target}
\end{equation}
where a transformer term is omitted when no corresponding transformer result is present. The physical AC-feasibility condition is
\begin{equation}
    f_{\mathrm{AC}}(x)\leq1.
    \label{eq:pandapower_feasible}
\end{equation}

The dataset consists of $20\,000$ points from a deterministic scrambled five-dimensional Sobol desig. Each point is evaluated once by the frozen AC power-flow simulation to generate its regression target and feasibility label; all $20\,000$ simulations in the retained dataset converged. The selected $[0,25]^5$ MW domain deliberately intersects the feasible boundary: approximately $49.25\%$ of the resulting designs satisfy \Eqref{eq:pandapower_feasible}. A single permutation assigns $12\,000$, $4\,000$, and $4\,000$ observations to training, validation, and test sets, and these exact splits are shared by all model seeds.

In contrast to the two analytic benchmarks, there is no dense high-dimensional ground-truth grid. Held-out prediction performance is therefore evaluated on the $4\,000$ independently reserved simulator-labeled
test designs. For downstream optimization, each returned candidate is converted back to MW and re-evaluated by the same frozen AC power-flow simulation; this simulator evaluation determines physical AC feasibility, whereas any global optimality certificate refers only to the corresponding learned constraint.


\subsection{\method architectures}
\label{app:trio_implementation}

\method{} is implemented as a minimum envelope of $Q$ ellipsoidal experts.
Each expert combines a learned center and positive-definite quadratic form with
a strictly increasing scalar radial function. This separation is the key
implementation property: the radial network is used while learning the scalar
response, but for any fixed target level it can be inverted once and removed,
leaving an explicit union of ellipsoids.

For expert $r$, let
\begin{equation}
    q_r(x)
    =
    (x-c_r)^\top A_r(x-c_r),
    \qquad
    A_r \succ 0,
    \label{eq:trio_quadratic}
\end{equation}
where $c_r$ is the learned center. We parameterize $A_r$ through a Cholesky
factor, which enforces positive definiteness throughout training. The generic
predictor takes the form
\begin{equation}
    F_\theta(x)
    =
    \min_{r=1,\ldots,Q}
    \left[
        \beta_r+\phi_r\!\left(d_r(x)\right)
    \right],
    \qquad
    d_r(x)=\sqrt{q_r(x)},
    \label{eq:trio_forward}
\end{equation}
where $\beta_r$ is an expert-specific offset and $\phi_r$ is constrained to be
strictly increasing. The experiments use two radial backbones, summarized in
Table~\ref{tab:trio_variants}.

\begin{table}[h!]
    \centering
    \begin{adjustbox}{max width=\textwidth}
    \begin{tabular}{l|l|c|c|l|c}
        \toprule
        \textbf{Benchmark}
        & \textbf{Radial backbone}
        & \textbf{$Q$}
        & \textbf{Units per region}
        & \textbf{Initialization}
        & \textbf{Parameters}
        \\
        \midrule

        Complex powers
        & Broken-Power
        & $128$
        & --
        & Standard
        & $1\,152$
        \\

        Six-Hump Camel
        & Neural
        & $\{16,32,64,128,256\}$
        & $8$
        & Mixed low-target FPS
        & $31\times Q$
        \\

        Pandapower
        & Neural
        & $128$
        & $32$
        & Mixed low-target FPS
        & $15\,104$
        \\

        \bottomrule
    \end{tabular}
    \end{adjustbox}
    \caption{
        \method{} configurations used in the experiments.
        The Broken-Power model is used for the complex-power benchmark,
        whereas the neural radial backbone is used for Six-Hump
        Camel and Pandapower. 
    }
    \label{tab:trio_variants}
\end{table}

\paragraph{Exact compilation.}
For a target level $g$, only experts with
\begin{equation}
    \beta_r\leq g
\end{equation}
can contribute to the sublevel set. Since $\phi_r$ is strictly increasing,
each active expert admits a unique radius
\begin{equation}
    R_r(g)
    =
    \phi_r^{-1}\!\left(g-\beta_r\right).
    \label{eq:trio_radial_inverse}
\end{equation}
Consequently,
\begin{equation}
    \left\{
        x:F_\theta(x)\leq g
    \right\}
    =
    \bigcup_{r:\beta_r\leq g}
    \left\{
        x:
        (x-c_r)^\top A_r(x-c_r)
        \leq R_r(g)^2
    \right\}.
    \label{eq:trio_compilation}
\end{equation}
Thus, radial inversion is performed only once per active expert and target.
Afterwards, preimage membership and downstream optimization use the detached
quadratic representation in \Eqref{eq:trio_compilation}; no radial-network
evaluation is required.

For the neural radial laws, the inverse in
\Eqref{eq:trio_radial_inverse} is obtained by a one-dimensional bracketed
root solve. In the Pandapower implementation, the tolerance is $10^{-12}$.
The complex-power Broken-Power backbone is likewise strictly monotone and is
inverted once for each active expert before projection.

\paragraph{Radial backbones.}
For the complex-power benchmark, we use the Broken-Power backbone with
$Q=128$ experts and standard initialization. This is the same model for all
powers $p\in\{4,6,8,10,12\}$ and contains $1\,152$ trainable parameters.

For Six-Hump Camel, each expert instead uses an anchored monotone Wide-Tanh
radial map with eight Tanh units per region,
\begin{equation}
    \phi_r(d)
    =
    a_r d
    +
    \sum_{j=1}^{8}
    v_{rj}
    \left[
        \tanh\!\left(
            w_{rj}(d-\kappa_{rj})
        \right)
        -
        \tanh\!\left(
            -w_{rj}\kappa_{rj}
        \right)
    \right],
    \label{eq:trio_wide_tanh_six_hump}
\end{equation}
where $a_r$, $v_{rj}$, and $w_{rj}$ are constrained to be positive.
The subtraction term anchors the radial map at
$\phi_r(0)=0$, while the positive linear component prevents saturation and
ensures strict monotonicity. The implementation uses the numerically
stabilized radial distance
\begin{equation}
    d_{\varepsilon,r}(x)
    =
    \sqrt{q_r(x)+\varepsilon}-\sqrt{\varepsilon},
    \qquad
    \varepsilon=10^{-12}.
    \label{eq:trio_stable_distance}
\end{equation}
Hence, if the inverse radial value is $R_r$, the corresponding detached
quadratic constraint is
\begin{equation}
    q_r(x)
    \leq
    \left(R_r+\sqrt{\varepsilon}\right)^2-\varepsilon.
    \label{eq:trio_stable_radius}
\end{equation}

The Pandapower model uses the same monotone neural-radial principle with $32$ Tanh units per expert. For
$v_r=\log(1+q_r)$, its expert score is
\begin{equation}
    \beta_r
    +
    a_r v_r
    +
    \sum_{j=1}^{32}
    w_{rj}
    \left[
        \tanh(s_{rj}v_r+t_{rj})
        -
        \tanh(t_{rj})
    \right],
    \label{eq:trio_wide_tanh_pandapower}
\end{equation}
with $a_r$, $w_{rj}$, and $s_{rj}$ constrained to be positive.
This parameterization is again anchored and strictly increasing in the
quadratic radius. Once the scalar inverse $v_r^\star$ has been found, the
compiled ellipsoid radius follows directly from
\begin{equation}
    q_r(x)\leq \exp(v_r^\star)-1.
    \label{eq:trio_pandapower_radius}
\end{equation}

\paragraph{Mixed low-target initialization.}
For the neural Wide-Tanh variants, we use a training-data-only initialization
designed to place experts both throughout the input domain and in regions of
small target values. Half of the $Q$ centers are selected by deterministic
farthest-point sampling (FPS) over the complete normalized training set. The
remaining half are selected by FPS restricted to
\begin{equation}
    \mathcal{D}_{\mathrm{low}}
    =
    \left\{
        (x_i,y_i):
        y_i\leq q_{0.1}(y_{\mathrm{train}})
    \right\},
    \label{eq:trio_low_target_set}
\end{equation}
where $q_{0.1}$ denotes the empirical $10\%$ target quantile. Given the
previously selected centers, FPS iteratively chooses
\begin{equation}
    c_{k+1}
    =
    \arg\max_{x_i\in\mathcal{D}}
    \min_{j\leq k}
    \|x_i-c_j\|_2^2.
    \label{eq:trio_fps}
\end{equation}
For every selected training point $x_{i(r)}$, we initialize
\begin{equation}
    c_r=x_{i(r)},
    \qquad
    \beta_r=y_{i(r)}.
    \label{eq:trio_beta_initialization}
\end{equation}
The procedure uses only the training split and requires no analytic knowledge
of the underlying data-generating process or simulator. It is used for the
Six-Hump Camel and Pandapower experiments, while the complex-power experiments
retain the standard Broken-Power initialization.

\clearpage

\subsection{Neural baselines}
\label{app:baselines}

We compare \method against standard unconstrained neural predictors and an input convex neural network (ICNN). Whenever architectures are compared at a fixed capacity, baseline widths are chosen to match the number of trainable parameters of \method as closely as possible. The CPWL MLP uses ReLU activations, while the Smooth MLP uses Tanh activations. The ICNN uses Softplus activations and structurally nonnegative hidden-to-hidden and output weights to preserve convexity of its scalar output. The architectures are summarized in Tables~\ref{tab:baseline_architectures} and~\ref{tab:six_hump_parameter_matching}.

\begin{table}[h!]
    \centering
    \begin{adjustbox}{max width=\textwidth}
    \begin{tabular}{l|l|l|c}
        \toprule
        \textbf{Benchmark}
        & \textbf{Model}
        & \textbf{Architecture}
        & \textbf{Parameters}
        \\
        \midrule

        \multirow{4}{*}{Complex powers}
        & \method{}
        & $Q=128$ Broken-Power experts
        & $1\,152$
        \\
        & CPWL MLP
        & $2\rightarrow24\rightarrow42\rightarrow1$, ReLU
        & $1\,165$
        \\
        & Smooth MLP
        & $2\rightarrow24\rightarrow42\rightarrow1$, Tanh
        & $1\,165$
        \\
        & ICNN
        & two hidden layers, width $31$, Softplus
        & $1\,212$
        \\
        \midrule

        \multirow{4}{*}{Pandapower}
        & \method{}
        & $Q=128$ Wide-Tanh experts
        & $15\,104$
        \\
        & CPWL MLP
        & $5\rightarrow117\rightarrow121\rightarrow1$, ReLU
        & $15\,102$
        \\
        & Smooth MLP
        & $5\rightarrow117\rightarrow121\rightarrow1$, Tanh
        & $15\,102$
        \\
        & ICNN
        & two hidden layers, width $116$, Softplus
        & $15\,086$
        \\

        \bottomrule
    \end{tabular}
    \end{adjustbox}
    \caption{
        Parameter-matched predictor architectures for the complex-power and
        Pandapower experiments. Baseline widths are selected to closely match
        the trainable parameter count of \method{}.
    }
    \label{tab:baseline_architectures}
\end{table}

For the Six-Hump Camel capacity study, the parameter budget varies with the
number of \method{} experts. The matched CPWL MLP and ICNN architectures are
therefore adjusted separately at each value of $Q$ (see Table~\ref{tab:six_hump_parameter_matching}).

\begin{table}[h!]
    \centering
    \begin{adjustbox}{max width=\textwidth}
    \begin{tabular}{c|cc|cc|cc}
        \toprule
        \multirow{2}{*}{$Q$}
        & \multicolumn{2}{c|}{\method{}}
        & \multicolumn{2}{c|}{CPWL MLP}
        & \multicolumn{2}{c}{ICNN}
        \\
        \cmidrule(lr){2-3}
        \cmidrule(lr){4-5}
        \cmidrule(lr){6-7}
        & Architecture & Params.
        & Hidden widths & Params.
        & Hidden width & Params.
        \\
        \midrule

        $16$
        & $16$ experts & $496$
        & $(16,25)$ & $499$
        & $19$ & $516$
        \\
        $32$
        & $32$ experts & $992$
        & $(23,37)$ & $995$
        & $28$ & $1\,011$
        \\
        $64$
        & $64$ experts & $1\,984$
        & $(37,48)$ & $1\,984$
        & $41$ & $2\,012$
        \\
        $128$
        & $128$ experts & $3\,968$
        & $(54,68)$ & $3\,971$
        & $59$ & $3\,956$
        \\
        $256$
        & $256$ experts & $7\,936$
        & $(73,103)$ & $7\,945$
        & $85$ & $7\,908$
        \\

        \bottomrule
    \end{tabular}
    \end{adjustbox}
    \caption{
        Parameter matching for the Six-Hump Camel capacity experiment.
        The CPWL MLP has two ReLU hidden layers with the listed widths; the
        ICNN has two Softplus hidden layers of the listed common width.
    }
    \label{tab:six_hump_parameter_matching}
\end{table}

\subsection{Training hyperparameters}
\label{app:training}

All predictors, including \method and the neural baselines, are trained by direct scalar mean-squared error using Adam in double precision. We use the same core optimization settings across all three benchmarks (see Table~\ref{tab:training_hyperparameters}).

\begin{table}[h!]
    \centering
    \begin{adjustbox}{max width=\textwidth}
    \begin{tabular}{l|c|c|c|c|c|c}
        \toprule
        \textbf{Setting}
        & Loss
        & Optimizer
        & Learning rate
        & Batch size
        & Precision
        & Validation interval
        \\
        \midrule
        \textbf{Configuration}
        & MSE
        & Adam
        & $10^{-3}$
        & $1{,}024$
        & float64
        & $500$ steps
        \\
        \bottomrule
    \end{tabular}
    \end{adjustbox}
    \caption{
        Shared training hyperparameters used across the experiments.
    }
    \label{tab:training_hyperparameters}
\end{table}

For \method{}, the initial optimization phase uses a differentiable soft
minimum over experts, with temperature cosine-annealed from $0.20$ to $0.01$
over the first $120{,}000$ steps. Training then switches to the literal hard
minimum in \Eqref{eq:trio_forward}. Validation always evaluates the hard
predictor, and all reported results use the checkpoint with the best validation
performance. For comparability, the baseline models use the same $120{,}000$-step shield before early stopping and checkpoint eligibility, while training their standard forward maps throughout.


\subsection{Preimage extraction and downstream optimization}
\label{app:inverse_optimization}

The experiments distinguish between the learned scalar predictor and the
procedure used to recover or optimize over its sublevel set. For \method{},
\Eqref{eq:trio_compilation} gives the complete learned preimage explicitly as
a finite union of ellipsoids. The neural baselines instead require a
post-hoc extraction or optimization procedure. Table~\ref{tab:inverse_methods}
summarizes the methods used in the experiments.

\begin{table}[h!]
    \centering
    \begin{adjustbox}{max width=\textwidth}
    \begin{tabular}{l|l|l|l}
        \toprule
        \textbf{Predictor}
        & \textbf{Inverse procedure}
        & \textbf{Used for}
        & \textbf{Learned-model guarantee}
        \\
        \midrule

        \method{}
        & Compiled ellipsoid union
        & Preimage extraction
        & Exact complete preimage
        \\

        \method{}
        & Ellipsoid-wise optimization
        & Downstream optimization
        & Global optimum over learned preimage
        \\
        \midrule

        CPWL MLP
        & PREMAP2
        & Preimage extraction
        & Certified inner approximation
        \\

        CPWL MLP
        & Gurobi
        & Downstream optimization
        & Global on \texttt{OPTIMAL}
        \\

        Smooth MLP
        & IPOPT
        & Downstream optimization
        & Local
        \\

        Smooth MLP
        & SCIP
        & Downstream optimization
        & Global on \texttt{optimal}
        \\

        ICNN
        & CLARABEL
        & Downstream optimization
        & Convex global solve on \texttt{optimal}
        \\

        \bottomrule
    \end{tabular}
    \end{adjustbox}
    \caption{
        Preimage extraction and downstream optimization procedures.
        Guarantees refer to the frozen learned predictor, not to the analytic
        data-generating process or the Pandapower AC equations.
    }
    \label{tab:inverse_methods}
\end{table}

\paragraph{Six-Hump preimage extraction.}
The quantitative extraction experiment uses the model parameter-matched at $Q=64$ for
$g=0.4$. For \method, the learned sublevel set is obtained directly from
\Eqref{eq:trio_compilation}; hence its model-relative extraction coverage is
exactly one. We compare this representation against PREMAP2 applied to the
parameter-matched CPWL MLP.

PREMAP2 is run in its input space splitting mode and only its certified inner
approximation is used. Since PREMAP2 operates as a classifier, the scalar
sublevel-set condition is represented by the two-output map
\begin{equation}
    G(x)
    =
    \begin{bmatrix}
        g-F_{\mathrm{CPWL}}(x) \\
        0
    \end{bmatrix},
    \label{eq:premap_two_logit}
\end{equation}
for which classification into the first output is equivalent to
$F_{\mathrm{CPWL}}(x)\leq g$. We use refinement budgets
\begin{equation}
    B\in\{64,128,256,512,1024,2048\}.
\end{equation}
For the certified inner set $I_B$ and the complete learned CPWL preimage
$S_{\mathrm{CPWL}}(g)$, extraction coverage is evaluated as
\begin{equation}
    C_B
    =
    \frac{|I_B|}
         {|S_{\mathrm{CPWL}}(g)|}
    =
    \operatorname{IoU}
    \!\left(
        I_B,S_{\mathrm{CPWL}}(g)
    \right),
    \label{eq:premap_coverage}
\end{equation}
where the second equality follows from the certified inclusion
$I_B\subseteq S_{\mathrm{CPWL}}(g)$. All reported certificates are
independently decoded and checked on the $1201\times1201$ evaluation grid.
The reported construction times measure once-per-target preimage extraction,
i.e., \method{} compilation or PREMAP2 refinement, rather than subsequent
point-membership queries.

\paragraph{Complex-power projection.}
For the complex-power benchmark, the downstream problem is the nearest-point
projection
\begin{equation}
    \min_x \|x-q\|_2^2
    \qquad
    \text{s.t.}
    \qquad
    F_\theta(x)\leq g,
    \quad
    \|x\|_2\leq1 .
    \label{eq:complex_projection}
\end{equation}
For every power $p$, we use three target levels whose true sublevel sets have probability masses $0.2$, $0.5$, and $0.8$ under uniform sampling from the unit disk. Twenty query points are drawn for each level conditional on being truly infeasible, yielding $60$ nontrivial projection problems per power. The same queries are used for every predictor and model seed.

For \method, the learned constraint is represented as a finite union of ellipsoids,
\begin{equation}
    \mathcal P_{F_\theta}(g)
    =
    \bigcup_{r\in\mathcal A(g)}
    E_r(g),
    \qquad
    E_r(g)
    =
    \left\{
        x:
        (x-c_r)^\top A_r(x-c_r)
        \leq R_r(g)^2
    \right\}.
\end{equation}
This representation enables direct optimization without evaluating the neural
predictor. For the complex-power projection problem,
\begin{equation}
    \min_x \|x-q\|_2^2
    \qquad
    \mathrm{s.t.}
    \qquad
    x\in\mathcal P_{F_\theta}(g),
\end{equation}
the global solution is obtained by solving the projection problem on each active ellipsoid independently and selecting the candidate with the smallest distance. Projection onto a single ellipsoid is solved analytically through a one-dimensional monotone root-finding problem derived from the KKT conditions. Thus, the complete projection reduces to a finite number of scalar solves over the active experts.

The CPWL MLP is encoded exactly and solved with Gurobi; the Smooth MLP is optimized both with deterministic eight-start local IPOPT and with global SCIP spatial branch-and-bound, with a $30$\,s per-query limit for SCIP. The ICNN sublevel set is convex and is represented in CVXPY and solved using CLARABEL.

Because the complex-power DGP is analytic, the true projection optimum is computed independently rather than using any learned model. Writing a query as $q=\rho e^{\mathrm{i}\psi}$, the boundary of each angular branch can be reduced to a one-dimensional polar problem; all branches, relevant endpoints, and stationary candidates are considered. This reference computation is used only for evaluation and is excluded from all optimizer timings.

\paragraph{Pandapower calibration and optimization.}
The Pandapower experiment optimizes total renewable generation under a learned
AC-security constraint,
\begin{equation}
    \max_{u\in[0,1]^5}
    25\sum_{j=1}^{5}u_j
    \qquad
    \text{s.t.}
    \qquad
    F_\theta(u)\leq\tau_\theta .
    \label{eq:pandapower_optimization}
\end{equation}
The threshold $\tau_\theta$ is calibrated separately for every trained
predictor using the validation set only. For a candidate threshold $\tau$, we
compute the false-feasible rate among truly AC-infeasible validation points,
\begin{equation}
    \operatorname{FFR}_{\mathrm{val}}(\tau)
    =
    \frac{
        \#\{i:
        f_{\mathrm{AC}}(x_i)>1,\,
        F_\theta(u_i)\leq\tau\}
    }{
        \#\{i:f_{\mathrm{AC}}(x_i)>1\}
    },
    \label{eq:pandapower_ffr}
\end{equation}
and choose the largest candidate threshold satisfying
$\operatorname{FFR}_{\mathrm{val}}(\tau)\leq10^{-3}$.
No test-set or downstream-optimization outcome is used for calibration.

For \method, the calibrated constraint again compiles to a union of
ellipsoids. The linear objective in \Eqref{eq:pandapower_optimization} is
optimized exactly over the intersection of each ellipsoid with the box
$[0,1]^5$, and the best solution across active experts is returned. When box constraints
are present, violated coordinates are fixed to their active bounds and the remaining free variables define a lower-dimensional ellipsoidal slice. The same closed-form solution is then applied recursively on the reduced
ellipsoid. The best feasible candidate across all active ellipsoids is the global optimum of the compiled learned union intersected with the design box.

The parameter-matched CPWL MLP is encoded exactly as a mixed-integer linear model and solved by Gurobi. The Smooth MLP is solved using both SCIP spatial branch-and-bound and deterministic eight-start IPOPT, while the ICNN is represented as a convex epigraph problem and solved using CLARABEL. SCIP uses a $30$\,s solver limit. IPOPT is treated as a local method irrespective of agreement among its starts.

Finally, every returned design is converted from normalized coordinates to MW and evaluated by the frozen Pandapower AC power-flow simulation. This evaluation determines physical AC feasibility. In contrast, the ``global'' solver labels in Table~\ref{tab:inverse_methods} refer exclusively to optimality with respect to the corresponding learned constraint.

\subsection{Evaluation conventions}
\label{app:evaluation_software}

All reported multi-seed results are aggregated over seeds $101$--$110$ using the arithmetic mean and sample standard deviation. Uncertainty values in tables and figures denote this sample standard deviation across seeds with an AMD Ryzen 7 Pro CPU and 32GB of RAM.

For preimage evaluation, membership is always evaluated against the relevant ground-truth or learned-model set using the independent geometry grids described in Section~\ref{app:dgp_data}. For \method{}, exactness refers to equality between the neural predictor and its compiled ellipsoidal representation, not to equality with the underlying analytic function or simulator. Similarly, optimization certificates refer to the frozen learned constraint, while physical feasibility in the Pandapower experiment is determined only after
re-evaluation by the AC power-flow simulator.

All experiments use Python implementations with double-precision arithmetic. The main learning experiments use PyTorch. The complex-power and Six-Hump experiments use NumPy/SciPy-based numerical utilities for geometry and
optimization. The Pandapower experiments additionally use the Pandapower AC power-flow solver. Downstream optimization uses Gurobi \citep{gurobi2026} for CPWL MLPs, SCIP \citep{bestuzheva2025scip} and
IPOPT \citep{wachter2006implementation} for Smooth MLP baselines, and CLARABEL \citep{goulart2026clarabel} through CVXPY for ICNNs \citep{amos2017icnn}. The PREMAP2 \citep{bjorklund2025premap2, zhang2025premap} comparison is executed in a separate Python environment due to its different dependency requirements. Details can be found in the attached code base.


\end{document}

%% file: math_commands.tex
\usepackage{amsmath,amsfonts,bm}

\def\eqref#1{equation~\ref{#1}}
\def\Eqref#1{Equation~\ref{#1}}

\def\1{\bm{1}}

\DeclareMathAlphabet{\mathsfit}{\encodingdefault}{\sfdefault}{m}{sl}
\SetMathAlphabet{\mathsfit}{bold}{\encodingdefault}{\sfdefault}{bx}{n}

